%% file: main.tex
\documentclass[plain]{colt2020} 

\title[Probabilistic Dimensional and Margin Complexity]{Approximate is Good Enough:\\ Probabilistic Variants of Dimensional and Margin Complexity}
\usepackage{times}

\coltauthor{%
 \Name{Pritish Kamath} \Email{pritish@ttic.edu}\\
 \Name{Omar Montasser} \Email{omar@ttic.edu}\\
 \Name{Nathan Srebro} \Email{nati@ttic.edu}\\
 \addr Toyota Technological Institute at Chicago, 6045 S Kenwood Ave, Chicago, IL 60637
}

\def\tdotoggle{1}
\input{my-jmlr-fix}
\input{preamble}
\input{macros}

\begin{document}

\maketitle
\vspace{-1cm}
\begin{abstract}%
  We present and study approximate notions of dimensional and margin complexity, which correspond to the minimal dimension or norm of an embedding required to {\em approximate}, rather then exactly represent, a given hypothesis class.  We show that such notions are not only sufficient for learning using linear predictors or a kernel, but unlike the exact variants, are also necessary. Thus they are better suited for discussing limitations of linear or kernel methods. 
\end{abstract}




\section{Introduction}

A possible approach to learning is to choose some feature map $\varphi(x)$, or equivalently some kernel $K(x,x'):=\inangle{\varphi(x),\varphi(x')}$, appropriate for the problem, and then reduce the problem of learning, to that of learning a linear predictor, or a low (Euclidean or Hilbert) norm linear predictor, with respect to this embedding.  Such an approach is often successful in practice, and is the basis of ``kernel methods''.  But what are the inherent limits of such an approach?  Are there easily learnable hypothesis classes that {\em cannot} be learnt using such an approach, or perhaps require many more samples for learning, no matter what feature map or kernel is used?  This classic question about the limits of kernel methods has been explored by, e.g.~\citet{ben2002limitations}, and has lead to the notions of dimensional and margin complexity of a hypothesis class--- these correspond to the minimal dimension and minimal norm (respectively) of a feature space sufficient to \emph{exactly} represent all hypotheses in the class as linear predictors (see precise definitions in \sectionref{sec:dc-mc}).  Dimensional and margin complexity have also been studied in communication complexity \citep[See e.g.,][]{forster2002smallest, forster2003estimating, sherstov2008halfspace, razborov2010sign}.  Questions about the limits of kernel methods have resurfaced in recent years, in the context of understanding the advantage of deep learning over kernel methods, and identifying hypothesis classes that are learnable by training a neural network (using an efficient and simple training procedure)  but that are not learnable, or at least not without many more samples, using any kernel or feature map \citep{allenzhu19resnets,allenzhu20backward,yehudai19power}.

While the standard notions of dimensional and margin complexity are {\em sufficient} for learning by reduction to linear learning, they might not be {\em necessary} for such an approach.  This is because these notions insist on a feature map that can be used to {\em exactly} represent all hypotheses in the class, without any errors.  But for learning, it is sufficient to only \emph{approximate} the hypotheses, up to a small error $\eps$.  Furthermore, once we allow small errors, we might want to consider {\em randomized} rather than {\em deterministic} feature maps or kernels.  This is not only a hypothetical possibility---examples of specific randomized feature maps and kernels include Random Fourier Features  \citep{rahimi08kitchen}, the Conjugate Kernel \citep{daniely2017sgd}, and the Neural Tangent Kernel at a random initialized neural network \citep{jacot18ntk}.  One might ask if such randomized approximate embedding are in fact more powerful, or whether perhaps they can always be de-randomized and made exact.  In this paper we establish (\theoremref{thm:dc-vs-probdc-sep}, combined with \theoremref{thm:lin-tilde-vs-dc}) that randomized approximate embedding are indeed more powerful: we show that learning {\em is} possible using a randomized feature map, even for a hypothesis class for which no {\em exact} low dimensional representation exists (i.e.~with a very high, or even infinite, dimensional complexity).  In order to truly understand the power of kernel methods and reduction to linear learning, we must therefore also allow for such randomized feature maps and kernels, and understand their power and limitations.

In this paper we propose and study relaxed notions of dimensional and margin complexity that (a) allow for randomized feature maps; and (b) can be shown to be not only sufficient, but also necessary for learning by reduction to linear or kernel methods, and so yield strong lower bounds on the power of such an approach.  In discussing approximation of a hypothesis class, we must consider the loss used, and we study both classification problems with respect to a hard (0/1) loss, as well as classification and regression with continuous losses such as the hinge and squared loss.

In order to be able to discuss a necessary condition for ``learning by reduction to linear or kernel methods'' we must precisely define what we mean by this phrase. We do so in \sectionref{sec:linear-learning}.  We consider both distribution-dependent and distribution-independent learning.  Correspondingly, we define both distribution-dependent and distribution-independent approximate dimensional and margin complexity (in \sectionref{sec:dc-mc}). Our complexity definitions are justified by showing how they are both necessary and sufficient (in a sense) for learning by reduction to kernel or linear methods.  We also show how the distribution-dependent approximate dimension complexity lower bounds linear and kernel learning in a very broad sense, and with respect to a generic loss function. In \sectionref{sec:dc-lowerbounds} we further show how this complexity measure can be lower bounded, in turn, by other well studied complexity measures, providing for a generic way of obtaining strong lower bounds on the power of kernel methods.  

Our generic lower bound approach mirrors, to a large extent, the lower bound on the sample complexity of kernel based learning in several recent papers exploring the power of deep learning versus kernel method \citep{allenzhu19resnets,allenzhu20backward,yehudai19power}.  We distil the approach to a crisp complexity measure, which simplifies making such lower bound claims on specific hypothesis classes, and can also lead to stronger statements---we demonstrate this by strengthening the lower bound and resolving an open question of \cite{yehudai19power}. Our lower bound is stated in terms of the Statistical Query dimension, as defined by \citet{blum94weakly}, making a concrete connection between these complexity measures (``dimensionalities''). Our treatment also highlights a potential deficiency of this approach: although we can establish lower bounds for learning w.r.t.~the squared loss, using the same technique to establish a strong lower bound on learning w.r.t.~the 0/1 loss would resolve a long-standing question in circuit complexity theory and thus seems much more difficult.

We emphasize that when we speak of ``linear learning'' we refer to learning by minimizing the loss over all linear predictors without any regularization, and when we refer to ``kernel learning'' or ``norm based learning'' we are specifically referring to constraining or regularizing the Euclidean or Hilbert norm of linear predictors.  Learning using regularized linear predictors with other regularizers can be much more powerful---e.g.~any (finite) hypothesis class can be optimally learned using $\ell_1$ regularized learning with a feature map with dimension corresponding to the cardinality of the hypothesis class.  But this is not much different than using the hypothesis class itself, and we cannot use the ``kernel trick'' in order to avoid an explicit representation and search over this very high dimensional feature space.  In this paper, we are only concerned with (low dimensional) unregularized and $\ell_2$ regularized (kernel based) learning.

Throughout the paper, we are not overly concerned with the precise dependence on the ``error parameter'' $\eps$.  Although we always explicitly note the dependence on $\eps$, we think of it as a small constant, perhaps $0.01$, and do not worry about factors which are polynomial in $\eps$.  In this paper, we only refer to learning and approximating {\em in expectation}---it is possible to define and relate approximating and learning {\em with high probability} instead, but we avoid doing so for notational simplicity. 

\paragraph{Notations.} We refer to hypothesis classes $\calH \subseteq \calY^\calX$ over a domain $\calX$ and label set $\calY$. When $\calY$ is $\bbR$ or $\sbit$, and $|\calX|$ and $|\calH|$ are finite, we associate $\calH$ with a matrix $M_{\calH} \in \bbR^{\calH \times \calX}$ defined as $M_{\calH}(h,x) := h(x)$.
We  consider {\em loss} functions of the form $\ell : \bbR \times \calY \to \bbR_{\ge 0}$. In particular, we consider the 0/1 loss $\lzeroone(\what{y},y) := \ind\Set{\what{y}y \le 0}$, margin loss $\lmargin(\what{y},y) := \ind\Set{\what{y}y \le 1}$ and hinge loss $\lhinge(\what{y},y) := \max \Set{0, 1 - \what{y} y}$ for binary labels $\calY=\sbit$, and the squared loss $\lsq(\what{y},y) := \frac12 (\what{y} - y)^2$, for $\calY \subseteq \bbR$.
A loss $\ell$ is said to be $L$-Lipschitz if $|\ell(a,y) - \ell(a',y)| \le L |a-a'|$ for all $a, a' \in \bbR$ and $y \in \calY$.

We view learning algorithms as operating on a set of samples $S = \Set{(x_1, y_1), \ldots, (x_m,y_m)}$ drawn i.i.d. from a distribution $\scrD$ over $\calX \times \calY$. We say that $\scrD$ is {\em realizable} w.r.t. a hypothesis class $\calH$, to mean that $(x,y)\sim\scrD$ is sampled by first sampling $x \sim \calD$ (for some $\calD$) and setting $y = h_*(x_i)$ for some $h_* \in \calH$. We always use $\scrD$ to denote a distribution over $\calX \times \calY$ and $\calD$ to denote its marginal over $\calX$.
The {\em population loss} of a predictor $g : \calX \to \bbR$ w.r.t.~a loss $\ell$ is $\err_{\scrD}^{\ell}(g) := \Ex_{(x,y) \sim \scrD}\ \ell(g(x),y)$ whereas its {\em empirical loss} is $\err_S^{\ell}(g) :=~ \frac{1}{|S|} \sum_{(x,y) \in S} \ \ell(g(x),y)$. If $\scrD$ is realizable and sampled as $(x,h(x))$ with $x \sim \calD$ and $h \in \calH$, we define an alternate notation for $\err_{\scrD}^{\ell}(g)$ as $\err_{\calD,h}^{\ell}(g) := \Ex_{x \sim \calD} \insquare{\ell(g(x), h(x))}$.



\section{Dimension \& Margin Complexities and their Probabilistic Variants}\label{sec:dc-mc}

We recall the definitions of the dimension and margin complexities of a hypothesis class and introduce their probabilistic variants. Our definitions of the error-free notions are also stated in terms of a loss function so that we can then extend them to allow errors.

\subsection{Dimension Complexity}

\begin{definition}\label{def:det-dc}
Fix a hypothesis class $\calH \subseteq \calY^{\calX}$ and a loss $\ell$. The {\em dimension complexity} $\dc^{\ell}(\calH)$ is the smallest $d$ for which there exists an embedding $\varphi:\calX \to \bbR^d$ and a map $w : \calH \to \bbR^d$ such that for all $h \in \calH$ and $x \in \calX$, it holds that $\ell(\inangle{w(h), \varphi(x)}, h(x)) = 0$.
\end{definition}

\noindent 
For classification problems ($\calY=\sbit$) our definition coincides with the standard definition of dimensional complexity (equivalent to $\signrank(M_{\calH})$) for $\ell = \lzeroone$, and we will denote $\dc(\calH):=\dc^{\lzeroone}(\calH)$.  For finite hypothesis classes we also have $\dc(\calH)=\dc^{\lmargin}(\calH)=\dc^{\lhinge}(\calH)$.  For regression problems ($\calY = \bbR$), e.g. with the $\lsq$ loss, $\dc^{\lsq}(\calH)$ coincides with $\rank(M_{\calH})$.

\begin{definition} \label{def:prob-dc}
Fix a hypothesis class $\calH \subseteq \calY^{\calX}$, a loss $\ell$ and a parameter $\eps \ge 0$.
\begin{description}
\item [Probabilistic Distributional Dimension Complexity.] 
$\dc_{\eps}^{\calD,\ell}(\calH)$, parameterized by a distribution $\calD$ over $\calX$, is the smallest $d$ for which there exists a distribution $\calP$ over embeddings $\varphi:\calX \to \bbR^d$ such that for all $h\in\calH$,
\begin{equation}\label{eqn:dc-def}
\Ex\limits_{\varphi \sim \calP} \insquare{ \inf_{w \in \bbR^d} \ \err_{\calD,h}^{\ell}(\inangle{w, \varphi(\cdot)})} \le \eps\,.
\end{equation}
\item [Probabilistic Dimension Complexity.]
$\dc_{\eps}^{\ell}(\calH)$ is the smallest $d$ for which there exists a distribution $\calP$ over embeddings $\varphi:\calX \to \bbR^d$ such that for all distributions $\calD$ over $\calX$ and all $h\in\calH$, \equationref{eqn:dc-def} above holds.
\end{description}

\end{definition}
Again, for classification $\calY = \sbit$ we denote $\dc_{\eps}(\calH)=\dc_{\eps}^{\lzeroone}(\calH)$ and  $\dc_{\eps}^{\calD}(\calH)=\dc_{\eps}^{\calD,\lzeroone}(\calH)$, and at least for finite hypothesis classes these also agree with the complexities with respect to losses $\lmargin$ and $\lhinge$. Note that $\dc_{\eps}^{\ell}(\calH)$ is different from simply $\sup_{\calD} \dc_{\eps}^{\calD,\ell}(\calH)$. In particular, note the difference in order of quantifiers.
\begin{center}
\begin{tikzpicture}
\node (dc) at (-3.5,0) {$\dc_{\eps}^{\ell}(\calH) ~:=~ \min \ d$};
\node at ([shift={(-1.8,-0.6)}]dc) {$\exists \calP$};
\node at ([shift={(-0.8,-0.6)}]dc) {$\forall \calD$};
\node at ([shift={(0.2,-0.6)}]dc) {$\forall h$};
\node at ([shift={(1.4,-0.63)}]dc) {$\exists w | \varphi, h$};
\node[rectangle, rounded corners=5pt, below, minimum height=1.6cm, minimum width=5cm, draw=black!50, thick] at ([shift={(0,0.5)}]dc) {};

\node (dcD) at (3.5,0) {$\sup_{\calD} \dc_{\eps}^{\calD,\ell}(\calH) ~:=~ \min \ d$};
\node at ([shift={(-1.8,-0.6)}]dcD) {$\forall \calD$};
\node at ([shift={(-0.8,-0.6)}]dcD) {$\exists \calP$};
\node at ([shift={(0.2,-0.6)}]dcD) {$\forall h$};
\node at ([shift={(1.5,-0.63)}]dcD) {$\exists w | \varphi, h$};
\node[rectangle, rounded corners=5pt, below, minimum height=1.6cm, minimum width=5cm, draw=black!50, thick] at ([shift={(0,0.5)}]dcD) {};
\end{tikzpicture}
\end{center}

\subsection{Margin Complexity}

Margin complexity is defined in terms of embeddings $\varphi : \calX \to \bbH$, for any Hilbert space $\bbH$, thereby also allowing infinite dimensional embeddings, typically represented via a kernel $K_\varphi(x,x'):=\inangle{\varphi(x),\varphi(x')}_{\bbH}$. The sup-norm of the embedding is defined as $\|\varphi\|_{\infty} := \sup_{x \in \calX} \|\varphi(x)\|_{\bbH} = \sup_{x\in\calX} \sqrt{K_\varphi(x,x)}$. For a parameter $R \in \bbR_{\ge 0}$, let $\calB(\bbH; R) := \Set{w \in \bbH : \|w\|_{\bbH} \le R}$ be a norm ball of radius $R$ in the Hilbert space.

\begin{definition}\label{def:det-mc}
Fix a hypothesis class $\calH \subseteq \calY^{\calX}$ and a loss $\ell$. The {\em margin complexity} $\mc^{\ell}(\calH)$ is the smallest $R$ for which there exists an embedding $\varphi:\calX \to \bbH$ and a map $w : \calH \to \bbH$ with $\|\varphi\|_{\infty} \le 1$ and $\|w\|_{\infty} \le R$ such that for all $h \in \calH$ and $x \in \calX$, it holds that $\ell(\inangle{w(h), \varphi(x)}, h(x)) = 0$.
\end{definition}

\noindent
This definition does not make sense for the $\lzeroone$ loss, since $\lzeroone$ is scale-invariant. However, in the case of $\calY = \sbit$, it coincides with the standard definition of margin complexity for the margin loss $\lmargin$ (and hinge loss $\lhinge$), and we denote $\mc(\calH):=\mc^{\lmargin}(\calH)$.  For the squared loss $\lsq$, the definition coincides with the $\gamma_{2:\ell_1\rightarrow\ell_\infty}$ norm \citep{jameson_1987}, a.k.a.~the ``max norm'' \citep{srebro2005rank}.  Especially with a general loss function, ``$\mc$'' is really a form of ``norm-complexity'', but we still refer to it as ``margin complexity'' and use $\mc$ since it does capture the (inverse) margin when $\ell = \lmargin$ and this term is already widely used in the literature.

\begin{definition} \label{def:prob-mc}
Fix a hypothesis class $\calH \subseteq \calY^{\calX}$, a loss $\ell$ and a parameter $\eps \ge 0$. 
\begin{description}
\item [Probabilistic Distributional Margin Complexity.]
$\mc_{\eps}^{\calD,\ell}(\calH)$, parameterized by a distribution $\calD$ over $\calX$, is the smallest $R$ for which there exists a distribution $\calP$ over embeddings $\varphi:\calX \to \bbH$ with $\|\varphi\|_{\infty} \le 1$ such that for all $h\in\calH$,
\begin{equation}\label{eqn:mc-def}
\Ex\limits_{\varphi \sim \calP} \insquare{\inf_{w \in \calB(\bbH; R)} \ \err_{\calD,h}^{ \ell}(\inangle{w, \varphi(\cdot)})} \le \eps\,.
\end{equation}
\item [Probabilistic Margin Complexity.]
$\mc_{\eps}^{\ell}(\calH)$ is the smallest $R$ for which there exists a distribution $\calP$ over embeddings $\varphi:\calX \to \bbH$ with $\|\varphi\|_{\infty} \le 1$ such that for all distributions $\calD$ over $\calX$ and all $h\in\calH$, \equationref{eqn:mc-def} above holds.
\end{description}
\end{definition}
When $\calY = \sbit$, we denote $\mc_{\eps}(\calH)=\mc_{\eps}^{\lmargin}(\calH)$ and  $\mc_{\eps}^{\calD}(\calH)=\mc_{\eps}^{\calD,\lmargin}(\calH)$.

\subsection{Relationship between Probabilistic Dimension \& Margin Complexity}

A classic result attributed to \cite{arriaga99algorithmic} and \cite{ben2002limitations} shows that
\begin{equation}\label{eqn:dc-mc}
\dc(\calH) \le \mc(\calH)^2 \cdot \calO(\log |\calH||\calX|)\,.
\end{equation}
This result is proved by an application of the lemma of \cite{johnson1984extensions}. The term of $\calO(\log |\calH||\calX|)$ comes up due to a union bound over all pairs of $(x,h) \in \calX \times \calH$. Although the result can be seen as establishing a tight connection between the dimension and margin complexity, it is not applicable with continuous (or simply infinite) domains, and we are not aware of any way of avoiding this dependence on the cardinality of the domain.  

As a first application of our probabilistic notions, we show how this bypasses the cardinality dependence when allowing a randomized feature map.
\begin{lemma}[\boldmath Relating $\dc$ and $\mc$]\label{lem:dc-mc}
For all $\calH \subseteq \calY^{\calX}$ and parameters $\eps, \eta > 0$,\\[-7mm]
\begin{align*}
\text{(i) } \dc_{\eps+\eta}(\calH) &~~\le~~ \mc_{\eps}(\calH)^2 \cdot \calO\inparen{\log (1/\eta)},\\[1mm]
\text{(ii) } \dc_{\eps+\eta}^{\ell}(\calH) &~~\le~~ \mc_{\eps}^{\ell}(\calH)^2 \cdot \calO\inparen{L/\eta}^2 \quad \text{for any $L$-Lipschitz loss $\ell$, and}\\[1mm]
\text{(iii) } \dc_{\eps+\eta}^{\lsq}(\calH) &~~\le~~ \mc_{\eps}^{\lsq}(\calH)^2 \cdot \calO\inparen{(\eps+\eta)/\eta^2}.\\[-9mm]
\end{align*}
Analogous statements relating $\dc_{\eps + \eta}^{\calD,\ell}$ and $\mc_{\eps}^{\calD,\ell}$  hold as well for any distribution $\calD$ over $\calX$.
\end{lemma}
\noindent The proof is similar to that of \cite{ben2002limitations} in its use of the lemma of \cite{johnson1984extensions}. We defer the proof details to \appendixref{apx:proof-dc-mc}.  The random feature map used here is analogous to random features used in practice to approximate kernels \citep{rahimi07random}.

\subsection{Separations between Deterministic and Probabilistic Dimension Complexity}

We show that the probabilistic variants $\dc_{\eps}$ and $\dc_{\eps}^{\calD}$ can sometimes be significantly smaller than the classic notion of $\dc$. We show that dimension complexity can be exponentially larger than probabilistic dimension complexity (with respect to $\lzeroone$). Moreover, if we focus on the distributional version, then in fact dimension complexity can be ``infinitely larger'' than probabilistic distributional dimension complexity and moreover this separation holds for different losses such as $\lzeroone$, $\lsq$ and $\lhinge$, as well as for margin complexity.

\begin{theorem}[\boldmath Exponential Distribution Independent Gap ]\label{thm:dc-vs-probdc-sep}
For $\calX = \sbit^n$, there exists a hypothesis class $\calH \subseteq \sbit^\calX$ with $|\calH| = 2^n$ such that, for all $\eps \in (0,\nicefrac{1}{2})$,
\vspace{-0.2cm}
$$ \dc_{\eps}(\calH) ~\le~ \calO\inparen{n^4/\eps} \qquad \sAND \qquad \dc(\calH) ~\ge~ 2^{\Omega(n^{1/4})}.$$
\end{theorem}

\begin{theorem}[\boldmath ``Infinite'' Distribution Dependent Gap]\label{thm:dc-vs-prob-dist-dc-sep} For every $n$, there exist hypothesis classes $\calH \subseteq \sbit^{\calX}$ with $|\calH| = |\calX| = 2^n$ such that for all $\eps \in (0,\nicefrac{1}{2})$,
\vspace{-0.2cm}
$$\sup_{\calD}\ \dc_{\eps}^{\calD,\ell}(\calH) ~\leq~ \calO\inparen{1/\eps^2} \quad \sAND \quad \  \dc^{\ell}(\calH) ~\ge~ 2^{\Omega(n)} \qquad \textrm{for } \ell \in \Set{\lzeroone, \lsq, \lhinge}$$
\vspace{-5mm}
$$\sup_{\calD}\ \mc_{\eps}^{\calD,\ell}(\calH) ~\leq~ \calO\inparen{1/\eps^2} \ \quad \sAND \quad \mc^{\ell}(\calH) ~\ge~ 2^{\Omega(n)} \qquad \textrm{for }\ell \in \Set{\lmargin, \lsq, \lhinge}$$
\end{theorem}
\noindent We prove \theoremref{thm:dc-vs-probdc-sep} as follows (full details in \appendixref{apx:proof-dc-vs-probdc-sep}): We define another notion of probabilistic dimension complexity that has a stronger requirement of pointwise correctness and hence is larger than $\dc_{\eps}$. This notion is equivalent to {\em probabilistic sign-rank} studied in communication complexity. In particular, \cite{alman17probabilistic} showed that if the function $E_{\calH} : \calH \times \calX \to \sbit$ defined as $E_{\calH}(h,x) := h(x)$ is computable by a ``small'' depth-$2$ threshold circuit (for some encoding of $\calH$ and $\calX$ into bits), then $M_{\calH}$ has ``small'' probabilistic sign-rank. The theorem follows from a lower bound on sign-rank shown by \cite{chattopadhyay18short} for matrices that are computable by ``small'' depth-$2$ threshold circuits. The hypothesis class $\calH$ witnessing this separation is a class of {\em decision lists of conjunctions over disjoint variables}.

We prove \theoremref{thm:dc-vs-prob-dist-dc-sep} as follows (full details in \appendixref{apx:proof-dc-vs-prob-dist-dc-sep}): We use the ``covering lemma'' of \cite{haussler1995sphere} to show that the probabilistic distributional dimension complexity of any class can be bounded, albeit exponentially, in terms of the VC dimension, establishing the following Lemma:
\begin{lemma}[\boldmath $\dc_{\eps}^{\calD,\ell}$ and $\mc_{\eps}^{\calD,\ell}$ versus $\VCdim$]
\label{lem:dc_vc_upper}
There exists universal constants $c,K$ such that for all hypothesis classes $\calH \subseteq \sbit^{\calX}$, parameter $\eps > 0$ and all losses $\ell \in \Set{\lzeroone,\lsq,\lhinge}$ (in case of $\dc$) and $\ell \in \Set{\lmargin,\lsq,\lhinge}$ (in case of $\mc$),
\vspace{-0.4cm}
$$
\sup_{\calD} \dc_{\eps}^{\calD,\ell}(\calH)\ \ , \ \ \sup_{\calD} \mc_{\eps}^{\calD,\ell}(\calH)  ~\leq~ c \cdot \VCdim(\calH)\inparen{\frac{K}{\eps}}^{\VCdim(\calH)}.
$$
\end{lemma}
\noindent This is in contrast to the exact dimensional complexity, which can be polynomially large in $|\calH||\calX|$ even for classes of bounded VC dimension \cite{alon2016sign}. \theoremref{thm:dc-vs-prob-dist-dc-sep} now follows by considering a hypothesis class with VC-dimension $2$ with dimensional complexity of $2^{\Omega(n)}$.

The construction in \theoremref{thm:dc-vs-probdc-sep} uses extremely large magnitude features and weights, whereas the construction in \theoremref{thm:dc-vs-prob-dist-dc-sep} uses bounded magnitude of features and weights, but relies on having a known marginal $\calD$ over $\calX$. Our theorems therefore leave open the following questions.
\paragraph{Open Questions.} Is there an ``infinite'' separation between distribution independent $\dc_{\eps}$ and exact $\dc$? Is there a large (even finite) separation between distribution independent $\mc_{\eps}$ and exact $\mc$? Also between distribution independent $\dc_{\eps}^{\ell}$ and exact $\dc^{\ell}$ for $\ell \in \Set{\lsq, \lhinge}$? Can the distribution independent $\dc_{\eps}$ also be bounded in terms of the VC dimension?

\section{Linear \& Kernel Learnability with Probabilistic Embeddings} \label{sec:linear-learning}


We now turn to precisely defining learning by reduction to Linear Learning or Kernel Learning.  These notions serve as the primary motivation for our work, and their definitions guided the definitions of the other complexity notions we consider.

\remove{
Our goal is to characterize which hypothesis classes $\calH$ are learnable with linear learning and kernel learning methods. We show in the theorems below that our proposed complexity measures of probabilistic dimension and margin complexity upper bound the linear learning and the kernel learning complexity of $\calH$ respectively. An important point to note is that, as shown in Theorems~\ref{thm:dc-vs-probdc-sep} and \ref{thm:dc-vs-prob-dist-dc-sep}, these complexity measures can be much smaller than their deterministic variants which have been widely studied in literature before, and thus we have tighter learning guarantees.}

\subsection{Linear Learning Complexity}

\noindent Linear learning with a  feature map $\varphi : \calX \to \bbR^d$  boils down to relying on a learning rule of the form
\vspace{-0.1in}
\begin{equation}\label{eq:dcERM}
\textstyle \ERM_{\varphi}^{\ell}(S) ~:=~ \argmin_{w \in \bbR^d} \ \ \err_S^{\ell}(\inangle{w, \varphi(\cdot)})\,,
\end{equation}
where we require generalization for \emph{any} minimizer of the empirical error.  We formalize the {\em Linear Learning Complexity} of a hypothesis class $\calH$ as the minimal sample complexity of {\em any} learning rule of the form \eqref{eq:dcERM}.

\begin{definition} \label{def:lin-learn}
Fix a hypothesis class $\calH \subseteq \calY^{\calX}$, a loss $\ell$ and parameter $\eps > 0$.
\begin{description}
\item [Distributional Linear Learning Complexity] ${\Lin_{\eps}^{\calD,\ell}}(\calH)$, parameterized by a distribution $\calD$ over $\calX$, is the smallest $m$ for which there exists a distribution $\calP$ over embeddings $\varphi:\calX\to \bbR^{d}$ (for some $d \in \bbN$) such that for all realizable distributions $\scrD$ over $\calX \times \calY$ with marginal $\calD$ over $\calX$,
\begin{equation}\label{eqn:lin-learn}
\Ex_{\varphi\sim\calP}\ \Ex_{S\sim\scrD^m}\ \insquare{ \sup_{w\in \ERM_\varphi^{\ell}(S)} \ \err_{\scrD}^{ \ell}(\inangle{w, \varphi(\cdot)})} ~\leq~ \eps.
\end{equation}
\item [Linear Learning Complexity] ${\Lin_{\eps}^{\ell}}(\calH)$ is the smallest $m$ for which there exists a distribution $\calP$ over embeddings $\varphi:\calX\to \bbR^{d}$ (for some $d \in \bbN$) such that for all realizable distributions $\scrD$ over $\calX \times \calY$, \equationref{eqn:lin-learn} above holds.
\end{description}
For $\calY = \sbit$, we denote $\Lin_{\eps}(\calH)=\Lin_{\eps}^{\lzeroone}(\calH)$ and $\Lin_{\eps}^{\calD}(\calH)=\Lin_{\eps}^{\calD,\lzeroone}(\calH)$.
\end{definition}

\noindent To see more explicitly how low dimensional complexity is sufficient for linear learning, we also consider a stronger definition which requires that learning can be ensured by relying on linear dimension based generalization guarantees.  Recall that for a bounded or Lipschitz loss $\ell$ we have that for any distribution $\scrD$ \citep[c.f.][]{shalevshwartz14understanding},
\vspace{-0.2cm}
\begin{equation}\label{eqn:dc-gen}
    \Ex_{S \sim\scrD^m} \insquare{ \sup_{w\in\bbR^d} \inparen{ \err_{\scrD}^{ \ell}(\inangle{w, \varphi(\cdot)}) - \err_S^{\ell}(\inangle{w, \varphi(\cdot)})}} ~\leq~ C^\ell_\dc \sqrt\frac{d}{m}
\end{equation}
for some constant $C^\ell_\dc$ that depends on either the range or Lipschitz constant of the loss.  We note that the square-root dependence in the right-hand side can be improved to a nearly linear dependence when the empirical error is small, as it would be in our realizable setting.  This would yield a better polynomial dependence on the error parameter $\epsilon$.  Since we are less concerned here with the precise polynomial dependence on the error parameter, we refer only to the simpler uniform bound \eqref{eqn:dc-gen}.

\begin{definition}\label{def:glin-learn}
Fix a hypothesis class $\calH \subseteq \calY^{\calX}$, a loss $\ell$ that is either bounded or Lipschitz over the domain, and parameter $\eps > 0$.  The {\bf Guaranteed Linear Learning Complexity} ${\gLin_{\eps}^{\ell}}(\calH)$ and {\bf Distributional Guaranteed Linear Learning Complexity} ${\gLin_{\eps}^{\calD,\ell}}(\calH)$ are defined as in \definitionref{def:lin-learn}, but in terms of the smallest $m$ satisfying \equationref{eqn:lin-tilde-learn} below instead of \eqref{eqn:lin-learn},
\vspace{-0.2cm}
\begin{equation}\label{eqn:lin-tilde-learn}
\Ex_{\varphi\sim\calP}\ \Ex_{S\sim\scrD^m}\ \insquare{ \inf_{w \in \bbR^d} \ \err_S^{ \ell}(\inangle{w, \varphi(\cdot)})} + C^\ell_{\dc} \cdot \sqrt{\frac{d}{m}} ~\leq~ \eps,
\end{equation}
where $C^\ell_{\dc}$ is the loss-specific constant from \equationref{eqn:dc-gen}.
\end{definition}

\remove{\noindent With the said choice of $C_{\dc}^{\ell}$ it immediately follows via standard generalization bounds for halfspaces (cf. \cite{shalevshwartz14understanding}) that $\Lin_{\eps}(\calH) \leq \gLin_{\eps}(\calH)$. It also holds that $\gLin_{\eps}(\calH)$ is equivalent to $\dc_{\eps}(\calH)$ up to a polynomial dependence on $\eps$ (similarly with $\gLin^{\calD}_{\eps}(\calH)$ and $\dc^{\calD}_{\eps}(\calH)$).}

\begin{theorem}\label{thm:lin-tilde-vs-dc}
For any $\calH$, $\eps>0$ and Lipschitz or bounded loss $\ell$,
\vspace{-0.2cm}
$$
\Lin_{\eps}^{\ell}(\calH) ~\le~ \gLin_{\eps}^{\ell}(\calH) \qquad \sAND \qquad  \Omega\inparen{\frac{\dc_{\eps}^{\ell}(\calH)}{\eps^2}} ~\le~ \gLin_{\eps}^{\ell}(\calH)
~\le~ \calO\inparen{\frac{\dc_{\eps/2}^{\ell}(\calH)}{\eps^2}}
$$
and analogously for $\Lin_{\eps}^{\calD,\ell}$, $\gLin_{\eps}^{\calD,\ell}$ and $\dc_{\eps}^{\calD,\ell}$ and any distributions $\calD$ over $\calX$.
\end{theorem}

\noindent The proof of \theoremref{thm:lin-tilde-vs-dc} is presented in \appendixref{apx:proof-learn-dc-mc-upper}. Thus, $\dc_{\eps}(\calH)$ (and $\dc^{\calD}_{\eps}(\calH)$) precisely captures ``the sample complexity of learning $\calH$ using a linear embedding by relying on a guarantee that follows from dimension based generalization bounds'', and are therefore {\em sufficient} for linear learning. In \sectionref{subsec:learning-lower}, we will return to the question of whether they are also necessary for the weaker notion of linear learning of Definition \ref{def:lin-learn}, i.e.~whether they also lower bound $\Lin_{\eps}$ and $\Lin^\calD_\eps$.  But before that, we introduce the analogous notions for kernel based learning.

\subsection{Kernel Learning Complexity}

Recall that for any $\scrD$, any bounded embedding with $\|\varphi\|_\infty\leq 1$, any $R$ and any Lipschitz loss \citep[c.f.][]{shalevshwartz14understanding},
\vspace{-0.2cm}
\begin{equation}\label{eq:norm-based-uniform}
    \Ex_{S \sim\scrD^m} \insquare{ \sup_{w \in \calB(\bbH;R)} \left( \err_{\scrD}^{ \ell}(\inangle{w, \varphi(\cdot)}) - \err_S^{\ell}(\inangle{w, \varphi(\cdot)})\right)} ~\leq~ C^\ell_\mc \cdot \frac{R^2}{\sqrt{m}},
    \end{equation}
where $C^\ell_\mc$ is twice the Lipschitz constant,  which motivates the norm constrained ERM:
\begin{equation}
    \ERM_{\varphi}^{\ell}(S; R) ~:=~ \argmin_{w \in \calB(\bbH;R)} \ \ \err_S^{\ell}(\inangle{w, \varphi(\cdot)})\,.
\end{equation}
We therefore define the Kernel Learning Complexity and the Guaranteed Kernel Learning Complexity analogously to Definitions \ref{def:lin-learn} and \ref{def:glin-learn} but relying on $\ERM_{\varphi}^{\ell}(S; R)$.  We must be a bit more careful though, when considering margin based binary classification since neither the 0/1 error nor the margin error are Lipschitz.  We can still discuss the ERM w.r.t.~the margin loss, but can only use it to bound the population 0/1 loss.

\begin{definition}\label{def:ker-learn}
Fix a hypothesis class $\calH \subseteq \calY^{\calX}$, a Lipschitz loss $\ell$ and parameter $\eps > 0$.
\begin{description}
\item [Distributional Kernel Learning Complexity] $\Ker_{\eps}^{\calD,\ell}(\calH)$, parameterized by a distribution $\calD$ over $\calX$, is the smallest $m$ for which there exists a distribution $\calP$ over embeddings $\varphi : \calX \to \bbH$  with $\|\varphi\|_{\infty} \le 1$ and a parameter $R$ such that for all realizable distributions $\scrD$ over $\calX \times \calY$ with marginal $\calD$ over $\calX$,
\begin{equation}
    \label{eqn:ker-learn}
\Ex_{\varphi\sim\calP}\ \Ex_{S\sim\scrD^m}\ \insquare{ \sup_{w\in \ERM_{\varphi}^{\ell}(S;R)} \err_{\scrD}^{ \ell}(\inangle{w, \varphi(\cdot)})} ~\le~ \eps\,.
\end{equation}
\item [Kernel Learning Complexity] $\Ker_{\eps}^{\ell}(\calH)$ is the smallest $m$ for which there exists a distribution $\calP$ over embeddings $\varphi : \calX \to \bbH$ with $\|\varphi\|_{\infty} \le 1$ and a parameter $R$ such that for all realizable distributions $\scrD$ over $\calX \times \calY$, \equationref{eqn:ker-learn} above holds.
\end{description}
For $\calY=\sbit$ and $\ell = \lmargin$, we define $\Ker_{\eps}(\calH):=\Ker_{\eps}^{\lmargin}(\calH)$ and $\Ker_{\eps}^\calD(\calH):=\Ker_{\eps}^{\calD,\lmargin}(\calH)$ analogously, but require that \equationref{eqn:ker-learn-01} below holds instead of \eqref{eqn:ker-learn}:
\begin{equation}
\label{eqn:ker-learn-01}
\Ex_{\varphi\sim\calP}\ \Ex_{S\sim\scrD^m}\ \insquare{ \sup_{w\in \ERM_{\varphi}^{\lmargin}(S;R)} \err_{\scrD}^{ \lzeroone}(\inangle{w, \varphi(\cdot)})} ~\le~ \eps\,.
\end{equation}
\end{definition}
As we did in the case of linear learning, to relate $\Ker_{\eps}(\calH)$ to $\mc_{\eps}(\calH)$, we again consider a stronger notion that requires learning that can be guaranteed based only on the norm, using \equationref{eq:norm-based-uniform}:

\begin{definition}\label{def:gker-learn}
For a Lipschitz loss $\ell$, the {\bf Guaranteed Kernel Learning Complexity} ${\gKer_{\eps}^{\ell}}(\calH)$ and {\bf Distributional Guaranteed Kernel Learning Complexity} ${\gKer_{\eps}^{\calD,\ell}}(\calH)$ are defined as in \definitionref{def:lin-learn}, but in terms of the smallest $m$ satisfying \equationref{eqn:ker-tilde-learn} below instead of \eqref{eqn:ker-learn},
\vspace{-0.1cm}
\begin{equation}\label{eqn:ker-tilde-learn}
\Ex_{\varphi\sim\calP}\ \Ex_{S\sim\scrD^m}\ \insquare{ \inf_{w\in \calB(\bbH;R)} \err_S^{ \ell}(\inangle{w, \varphi(\cdot)})} + C_{\mc}^{\ell} \cdot \frac{B}{\sqrt{m}} ~\le~ \eps.
\end{equation}
For $\calY=\sbit$ and $\ell = \lmargin$, $\gKer_{\eps}(\calH)$ and $\gKer^\calD_{\eps}(\calH)$ are analogous but we require \equationref{eqn:ker-tilde-learn-01} holds instead:
\begin{equation}\label{eqn:ker-tilde-learn-01}
\Ex_{\varphi\sim\calP}\ \Ex_{S\sim\scrD^m}\ \insquare{ \inf_{w\in \calB(\bbH;R)} \err_S^{ \lmargin}(\inangle{w, \varphi(\cdot)})} + 2 \cdot \frac{B}{\sqrt{m}} ~\le~ \eps.
\end{equation}
\end{definition}

\begin{theorem}\label{thm:ker-tilde-vs-mc}
For any $\calH$, $\eps > 0$ and Lipschitz or bounded loss $\ell$,
\vspace{-0.2cm}
$$
\Ker_{\eps}^{\ell}(\calH) ~\le~ \gKer_{\eps}^{\ell}(\calH) \qquad \sAND \qquad 
\Omega\inparen{\frac{\mc_{\eps}^{\ell}(\calH)^2}{\eps^2}}
~\le~ \gKer_{\eps}^{\ell}(\calH) ~\le~
\calO\inparen{\frac{\mc_{\eps/2}^{\ell}(\calH)^2}{\eps^2}}
$$
and analogously for $\Ker_{\eps}^{\calD,\ell}$, $\gKer_{\eps}^{\calD,\ell}$ and $\mc_{\eps}^{\calD,\ell}$ for all distributions $\calD$ over $\calX$.
\end{theorem}
The proof of \theoremref{thm:ker-tilde-vs-mc} is presented in \appendixref{apx:proof-learn-dc-mc-upper}.  Thus, $\mc_{\eps}(\calH)$ and $\mc_{\eps}^{\calD}(\calH)$ precisely captures ``the sample complexity of learning $\calH$ using a kernel with a guarantee that follows from norm based generalization bounds'', both for margin-based binary classification, and with respect to a Lipschitz loss.\\[-3mm]

\noindent {\em Remark.} Our definitions of $\Lin_{\eps}$ and $\Ker_{\eps}$ capture realizable learning. We can also consider agnostic variants where we allow any $\scrD$ and the right hand side of \eqref{eqn:lin-learn}, \eqref{eqn:lin-tilde-learn}, \eqref{eqn:ker-learn}, \eqref{eqn:ker-learn-01}, \eqref{eqn:ker-tilde-learn} and \eqref{eqn:ker-tilde-learn-01} changes to $\inf_{h\in\calH} \err_{\scrD}(h)+\eps$, for loss functions where this makes sense.  The lower bounds on learning of course still hold, and for typical loss functions, including those discussed in this work, we can still get upper bounds in terms of the approximate dimensional and margin complexities.

\subsection{Lower Bounds on Learning}\label{subsec:learning-lower}



\noindent We saw that $\dc_{\eps}(\calH)$ and $\mc_{\eps}(\calH)$ precisely capture $\gLin_{\eps}(\calH)$ and $\gKer_{\eps}(\calH)$ i.e. ``learning based on dimension or norm guarantees''. But what about $\Lin_{\eps}(\calH)$ and $\Ker_{\eps}(\calH)$?  Perhaps for specific feature maps, e.g. if the image $\varphi(\calX)$ is degenerate in special ways, $\ERM$ on linear predictors, or perhaps low norm predictors, could give learning guarantees with significantly less than $d$ or $R^2$ samples? Can we say that $\dc_{\eps}(\calH)$ and $\mc_{\eps}(\calH)$ also tightly capture $\Lin_{\eps}(\calH)$ and $\Ker_{\eps}(\calH)$? While we are not able to say this in the distribution-independent setting, we can prove lower bounds in terms of the distribution dependent notion $\dc_{\eps}^{\calD,\ell}(\calH)$.


\begin{theorem}
\label{thm:learn-lowerbound}
For all $\calH$, losses $\ell$, distributions $\calD$ over $\calX$ and $\eps > 0$,\\[-3mm]
\begin{equation*}
\Lin_{\eps}^{\ell}(\calH) ~\ge~ \Lin_{\eps}^{\calD,\ell}(\calH) ~\ge~ \dc_{\eps}^{\calD,\ell}(\calH) \quad\quad\textrm{and}\quad\quad
\Ker_{\eps}^{\ell}(\calH) ~\ge~ \Ker_{\eps}^{\calD,\ell}(\calH) ~\ge~ \dc_{\eps}^{\calD,\ell}(\calH)
\end{equation*}
\end{theorem}
This follows as a consequence of the {\em Representer Theorem}, which allows us to replace any high-dimensional embedding by an $m$ dimensional one that is obtained as the span of the embeddings of the samples from $\scrD$. The proof is presented in \appendixref{apx:proof-learn-dc-mc-upper}.\\

Since \theoremref{thm:learn-lowerbound} holds for any distribution $\calD$, the lower bound on distribution independent learning can also be stated as
\begin{equation}
\Lin_{\eps}^{\ell}(\calH),\ \Ker_{\eps}^{\ell}(\calH) ~\geq~ \sup_{\calD}\dc_{\eps}^{\calD}(\calH).
\end{equation}
This supremum, which following \theoremref{thm:learn-lowerbound} tightly characterizes $\sup_\calD \Lin_\eps^{\calD,\ell}(\calH)$, should not be confused with the distribution independent $\dc_\eps(\calH)$.  
We can view $\sup_{\calD}$ $\Lin_{\eps}^{\calD}(\calH)$ as corresponding to a semi-supervised learning model where we have unlimited amount of unlabeled data, from which we can infer $\calD$, and use it to decide on a distribution over embeddings $\varphi$.

\begin{figure}[t]
\centering
\begin{tikzpicture}[scale=0.9,transform shape]
\tikzset{factor/.style={myPurple}}

\def\xgap{2}
\def \ygap{2}
\node (dc) at (4*\xgap,0) {$\dc$};
\node (dc-eps) at (3*\xgap,0) {$\dc_{\eps}$};
\node (Lin-eps) at (2*\xgap,0) {$\Lin_{\eps}$};
\node (Lin-eps-D) at (\xgap,0) {$\Lin_{\eps}^{\calD}$};
\node (dc-eps-D) at (0,0.5*\ygap) {$\dc_{\eps}^{\calD}$};
\node (Ker-eps-D) at (1.3*\xgap,\ygap) {$\Ker_{\eps}^{\calD}$};
\node[outer sep=0pt, inner sep=0pt] (mc-eps-D) at (2*\xgap,1.5*\ygap) {$\mc_{\eps}^{\calD}$};
\node (Ker-eps) at (2.3*\xgap,\ygap) {$\Ker_{\eps}$};
\node (mc-eps) at (3.3*\xgap,\ygap) {$\mc_{\eps}$};
\node (mc) at (4.3*\xgap,\ygap) {$\mc$};
\node (vc) at (3.4*\xgap, 2*\ygap) {\scriptsize $\exp\inparen{\wtilde{\calO}\inparen{\VCdim \cdot \log\frac{1}{\eps}}}$};

\path[-latex,line width=0.7pt,black]
(dc-eps-D) edge[bend left=15] (vc.west)
(mc-eps-D) edge (vc.south west)
(dc-eps-D) edge (Lin-eps-D)
(Lin-eps-D) edge (Lin-eps)
(Lin-eps) edge (dc-eps)
(dc-eps) edge[line width=1.5pt,dashed] (dc)
(dc-eps-D) edge (Ker-eps-D)
(Ker-eps-D) edge (Ker-eps)
(Ker-eps-D) edge[factor] (mc-eps-D)
(Ker-eps) edge (mc-eps)
(mc-eps) edge (mc)
(Lin-eps-D) edge[bend left=30,factor] (dc-eps-D)
(Lin-eps-D) edge[factor] (Ker-eps-D)
(dc-eps) edge[factor] (mc-eps)
(dc) edge[factor,line width=1.5pt,dashed] (mc);
\end{tikzpicture}
\caption{A comparison of all measures introduced, specialized to $\lzeroone$/$\lmargin$. A solid arrow $A \rightarrow B$ denotes $A(\calH) \le B(\calH)$, a solid \textcolor{myPurple}{purple} arrow $A$ $\textcolor{myPurple}{\rightarrow}$ $B$ denotes $A(\calH) \le B(\calH)$ up to some change of parameter $\eps$ and some multiplicative factors (either $\poly(1/\eps)$ or $\log(|\calH||\calX|$ in case of $\dc$ $\textcolor{myPurple}{\rightarrow}$ $\mc$). $A \dashrightarrow B$ denotes $A(\calH) \le B(\calH)$ and that there exists a class $\calH$ for which $A(\calH) \ll B(\calH)$. If $A$ is a distribution-dependent measure and $B$ is a distribution independent measure, then an arrow from $A \rightarrow B$ is meant for all $\calD$.}
\label{fig:complexity-zoo}
\end{figure}

\paragraph{Alternate Learning Rules} The  learning rule we studied as a ``kernel method'' was to minimize the loss subject to a constraint on the norm, $\min \err_S^{\ell}(\inangle{w, \varphi(\cdot)})$ subject to $\|w\|_{\bbH} \le R$.  This is reasonable as it corresponds to our generalization bounds, but often in practice other Pareto-optimal choices are considered, such as the minimum norm zero error (i.e.~hard margin) predictor $\min  \|w\|_{\bbH}$ subject to $\err_S^{\ell}(\inangle{w, \varphi(\cdot)})=0$, or perhaps a more relaxed version, $\min  \|w\|_{\bbH}$ subject to $\err_S^{\ell}(\inangle{w, \varphi(\cdot)})\leq\eps$
or Tikhonov-type regularization $\min
 \err_S^{\ell}(\inangle{w, \varphi(\cdot)}) +\lambda \|w\|_{\bbH}$.  

All of the above are variants of $\argmin_{w\in \bbH} g(\err_S^{\ell}(\inangle{w, \varphi(\cdot)}), \|w\|_{\bbH})$ for some monotone function $g:\bbR \times \bbR \to \bbR \cup \{\infty\}$, and hence the Representer Theorem holds for all them. Thus $\dc^{\calD,\ell}_{\eps}(\calH)$ would continue to be a lower bound on $\Ker_{\eps}^{\calD,\ell}(\calH)$ for any variant of its definition based on any of the above learning rules.


\section{Lower bounds on Probabilistic Distributional Dimension Complexity}\label{sec:dc-lowerbounds}

In \theoremref{thm:learn-lowerbound}, we established that the sample complexity of learning a hypothesis class $\calH$ with dimension-based or kernel-based linear learning is lower bounded by its probabilistic distributional dimension complexity, $\dc_{\eps}^{\calD,\ell}(\calH)$. In this section, we prove lower bounds on $\dc_{\eps}^{\calD,\ell}(\calH)$ in the case of squared-loss and the zero-one loss, demonstrating the utility of our proposed complexity measures in characterizing the limitations of linear learning. 

\subsection{Probabilistic dimension complexity w.r.t. Square Loss}\label{sec:sqsq}

\paragraph{Notations.} For a distribution $\calD$ over $\calX$, for any $f : \calX \to \bbR$ and $g : \calX \to \bbR$ we define $\inangle{f, g}_{\calD} := \Ex_{x \sim \calD} f(x) g(x)$ and $\|f\|_{\calD} := \sqrt{\inangle{f,f}_{\calD}} = \sqrt{\Ex_{x\sim \calD} f(x)^2}$. We say that a hypothesis class $\calH \subseteq \bbR^{\calX}$ is {\em normalized} if $\|h\|_{\calD} = 1$ for every $h \in \calH$. For any subset of hypotheses $\calH' \subseteq \calH$, define its corresponding Gram matrix $G_{\calH'}^{\calD} \in \bbR^{|\calH'| \times |\calH'|}$ as $G^{\calD}_{\calH'}(h, h') := \inangle{h, h'}_{\calD}$. For any $M \in \bbR^{t \times p}$ with $t \le p$, we use $\sigma_1(M) \le \ldots \le \sigma_t(M)$ to denote its singular values. For any symmetric $M \in \bbR^{t \times t}$, we use $\lambda_1(M) \le \ldots \le \lambda_t(M)$ to denote its eigenvalues. We use $\lambda_{\min}(M)$ to mean $\lambda_1(M)$.

\begin{definition}[SQ dimension]\label{def:SQdim}
For a distribution $\calD$ over $\calX$, the {\em Statistical Query dimension} of a normalized hypothesis class $\calH \subseteq \bbR^{\calX}$, denoted as $\SQdim^{\calD}(\calH)$, is the largest $t$ for which there exist hypotheses $h_1, \ldots, h_t \in \calH$ such that $\inangle{h_i, h_j}_{\calD} \le 1/2t$ for each $i \ne j$.
\end{definition}

\noindent While the Statistical Query dimension is a well studied quantity in learning theory \citep{blum94weakly}, we introduce a new measure that is more suited to our goal of proving lower bounds on $\dc_{\eps}^{\calD,\lsq}$. This measure is lower bounded by $\SQdim^{\calD}(\calH)$, but in general can be much larger.

\begin{definition}[minEV dimension]\label{def:minEVdim}
For a distribution $\calD$ over $\calX$, the {\em min-Eigenvalue dimension} of a normalized hypothesis class $\calH$, denoted as $\minEVdim^{\calD}(\calH; \lambda)$, is the largest $t$ for which there exists a subset of hypotheses $H_t := \Set{h_1, \ldots, h_t} \in \calH$ such that $\lambda_{\min}(G^{\calD}_{H_t}) \ge \lambda$.
\end{definition}

\begin{proposition}\label{prop:SQdim-EVdim}
For all distributions $\calD$ over $\calX$ and all normalized hypothesis classes $\calH \subseteq \bbR^{\calX}$,\\ $\SQdim^{\calD}(\calH) = t$ implies that $\minEVdim^{\calD}\inparen{\calH; \nicefrac{1}{2}} \ge t$
\end{proposition}
\begin{proof}
Let $\calH_t = \Set{h_1, \ldots, h_t} \subseteq \calH$ such that $\inangle{h_i, h_j}_{\calD} \le 1/2t$. Thus, all off-diagonal entries of $G^{\calD}_{\calH_t}$ are at most $1/2t$ in magnitude, whereas all diagonal entries are $1$. It follows from \cite{gerschgorin31uber} ``circle theorem'' that all eigenvalues of $G^{\calD}_{\calH_t}$ are at least $1 - t/2t = 1/2$.
\end{proof}

\noindent {\em Remark.} More generally, we could define $\SQdim^{\calD}(\calH;\gamma)$ with respect to parameter $\gamma < 1$, as the largest $t$ for which there exist hypotheses $h_1, \ldots, h_t \in \calH$ such that $\inangle{h_i, h_j}_{\calD} \le \gamma$ for each $i \ne j$. \propositionref{prop:SQdim-EVdim} could then be $\SQdim(\calH;\gamma) = t$ implies that $\minEVdim(\calH; 1-t\gamma) \ge t$.

\begin{theorem}\label{thm:dcl2-lowerbound-EV}
For all $\eps > 0$, all distributions $\calD$ over $\calX$ and normalized hypothesis classes $\calH \in \bbR^{\calX}$, it holds for any $\lambda \in (2\eps, 1]$ that
$$
\dc_{\eps}^{\calD, \lsq}(\calH) ~\ge~ \inparen{1-\frac{2\eps}{\lambda}} \cdot \minEVdim^{\calD}(\calH; \lambda) 
$$
\end{theorem}
Observe that the bound becomes vacuous at $\eps = \frac{1}{2}$, and rightly so, because the zero function incurs a square loss of $1/2$ for any $h \in \calH$, since $\calH$ is a normalized hypothesis class. The constant $0$ function is realizable with an embedding of dimension $1$.

Our proof of \theoremref{thm:dcl2-lowerbound-EV} is inspired by the technique due to \cite{alon13approxrank} for lower bounding the ``approximate rank'' of a matrix that is well studied in communication complexity. We present the full proof in \appendixref{apx:proof-dcl2-lowerbound-EV}.
Combining \propositionref{prop:SQdim-EVdim} with \theoremref{thm:dcl2-lowerbound-EV} immediately gives us the following corollary. 
\begin{corollary}\label{cor:dcl2-lowerbound-SQ}
For all distributions $\calD$ over $\calX$ and normalized hypothesis classes $\calH \subseteq \bbR^{\calX}$,
$$
\dc_{\eps}^{\calD, \lsq}(\calH) ~\ge~ \inparen{1-4\eps} \cdot \SQdim^{\calD}(\calH)\,.
$$
\end{corollary}

\subsubsection{Applications of Theorem~\ref{thm:dcl2-lowerbound-EV}}

We now discuss some applications of our \theoremref{thm:dcl2-lowerbound-EV} and \corollaryref{cor:dcl2-lowerbound-SQ}.

\paragraph{Example 1 : Parities.}
Let $\calX_n = \sbit^n$ and $\calH_n^{\oplus} = \Set{ \chi_S(x) := \prod_{i\in S}x_i : S\subseteq [n]}$ be the class of all parity functions on $n$ bits.
Let $\calD$ be the uniform distribution over $\calX$. For any two distinct subsets $S, T\subseteq [n]$, we have that $\inangle{\chi_S, \chi_T}_{\calD} = 0$. Thus, $\SQdim^{\calD}(\calH_n^{\oplus}) = 2^n$. More strongly, we also have $\minEVdim(\calH_n^{\oplus}; 1) = 2^n$. Thus, from \theoremref{thm:dcl2-lowerbound-EV}, we get
$$
\dc_{\eps}^{\calD, \lsq}(\calH_n^{\oplus}) ~\ge~ (1-2\eps) \cdot 2^n\,.
$$

\paragraph{Example 2 : ReLU with bounded weights.} The {\em Rectified Linear Unit} is a popular activation function used in neural networks; given by $x \mapsto [x]_+ = \max\Set{x, 0}$. It was recently shown by \cite{yehudai19power} that random features cannot be used to learn (or even approximate) a single ReLU neuron (over standard Gaussian inputs in $\bbR^n$ with $\poly(n)$ weights) unless the number of features or the magnitude of the learnt coefficients are exponential in $n$. Using \corollaryref{cor:dcl2-lowerbound-SQ}, we are able to improve on this result by removing the restriction on the magnitude of learnt coefficients and obtain a lower bound simply on the number of random features required (this was conjectured to be possible by \cite{yehudai19power}).

Let $\calH_{n,W,B}^{\relu} := \Set{x \mapsto [\inangle{w,x} + b]_+ : w \in \bbR^n, b \in \bbR, \text{ s.t. } \|w\|_2 \le W, |b| \le B}$ be the class of all functions obtained as a ReLU applied on a linear function with bounded weights.

\begin{theorem}[Strengthens Thm 4.2 in \cite{yehudai19power}]\label{thm:relu}
For $\calD$ being the standard Gaussian distribution over $\bbR^n$, there exists a choice of $W \le O(n^3)$ and $B \le O(n^4)$, such that, for any $\eps < \nicefrac{1}{4}$ that\\[-5mm]
$$
\dc_{\eps}^{\calD,\lsq}(\calH_{n,W,B}^{\relu}) ~\ge~ \exp(\Omega(n))
$$
\end{theorem}
Our proof builds on a proposition from \cite{yehudai19power} and also follows the outline there quite closely. However, we believe that this way of presenting the proof is more insightful as it is modular, involving a lower bound on SQ-dimension. The details are deferred to \appendixref{apx:proof-relu}.

\paragraph{Example 3 : studied by \cite{allenzhu19resnets,allenzhu20backward}.} Recently, \cite{allenzhu19resnets,allenzhu20backward} exhibited functions classes that can provably be ``efficiently'' learnt using a neural network, but require ``large'' number of samples or run-time for any kernel method to learn with respect to square loss. In our terminology, the function classes they consider can be shown to have ``large'' $\dc_{\eps}^{\calD,\lsq}$ measure using \theoremref{thm:dcl2-lowerbound-EV} and \corollaryref{cor:dcl2-lowerbound-SQ}. Since, the function classes they consider are somewhat specialized, we skip the details.

\subsection{Probabilistic dimension complexity w.r.t. 0-1 loss}

In the previous subsection we considered regression problems, and learning with respect to the squared loss.  We now turn to the classification and learning with respect to the 0/1 loss. 

We prove a lower bound on the probabilistic distributional dimension complexity w.r.t. $\lzeroone$ loss for the class of all $1$-sparse predictors $\calH_n^{\mathrm{1\text{-}sp}} \subseteq \sbit^{\calX_n}$ for $\calX_n = \sbit^n$ defined as $\calH_n^{\mathrm{1\text{-}sp}} := \Set{h_i : \calX_n \to \sbit : i\in[n] \text{ and } h_i(x) = x_i}$.

\begin{theorem}\label{thm:dcl01-lowerbound-sparse}
Fix $\eps < 1/2$. For $\calD$ being the uniform distribution over $\calX_n = \sbit^n$ it holds that,
\[
\dc_{\eps}^{\calD}(\calH_n^{\mathrm{1\text{-}sp}}) \ge n \cdot \left( \frac{(1 - h(\eps))}{4 \log (16e / (1-h(\eps)))} \right) - o(n)
\]
where $h(q) := q \log_2\inparen{\frac{1}{q}} + (1-q) \log_2\inparen{\frac{1}{1-q}}$ is the binary entropy function.
\end{theorem}

\noindent In particular, we have that $\dc_{\eps}^{\calD}(\calH_n^{\mathrm{1\text{-}sp}}) \ge \Omega(n)$ for any $\epsilon<\frac{1}{2}$, while the bound rightly becomes vacuous at $\eps = \frac{1}{2}$.  Contrast this linear scaling with $n$ to the VC dimension of 1-sparse predictors $\VCdim(\calH_n)\leq \log n$, which implies sparse linear predictors are learnable, using a direct approach, which only $O(\log n)$ samples. Thus, \theoremref{thm:dcl01-lowerbound-sparse} establishes that linear or kernel-based learning would require exponentially more samples than a direct approach.

\theoremref{thm:dcl01-lowerbound-sparse} also shows that the exponential dependence in our upper bound of $\dc_{\eps}^{\calD}(\calH)$ in terms of $\VCdim(\calH)$ (\lemmaref{lem:dc_vc_upper}) is indeed necessary, and \lemmaref{lem:dc_vc_upper} is, in this sense, tight. 

The key technique used in the proof of \theoremref{thm:dcl01-lowerbound-sparse} is the fact that random $n \times n$ sign-matrices require a sign-rank of $\Omega(n)$ to be even approximated on a constant ($> 1/2$) fraction of the entries. We partition the $n \times 2^n$ sign matrix $M_{\calH_n^{\mathrm{1\text{-}sp}}}$ randomly into blocks of $n \times n$ matrices and argue that most of those blocks must incur large error if the dimension of the embedding is small. The proof details are deferred to \appendixref{apx:proof-dcl01-lowerbound-sparse}.

\subsubsection{A Complexity-Theoretic Barrier}

In \theoremref{thm:dcl01-lowerbound-sparse} we proved a lower bound on $\dc_{\eps}^{\calD}(\calH_n)$ for the class of $1$-sparse predictors, which has $|\calX_n| = 2^{|\calH_n|}$. 
Even just representing a single instance in this example requires $\log \abs{\calX_n}=n$ bits, and so the {\em runtime} for any learning algorithm would also be at least $\Omega(n)$.  That is, even though we showed the sample complexity for linear or kernel based learning is exponential in the VC-dimension, i.e.~insisting on linear or kernel based learning causes an exponential increase in sample complexity, the sample complexity of linear learning is still no more than linear in the {\em runtime} or even {\em memory} of a direct approach.  This is in contrast to the examples of \sectionref{sec:sqsq}, where the lower bound on the sample complexity of linear or kernel based learning was exponential also in the representational cost of instances, i.e.~in $\log\abs{\calX}$.

Can we prove such a stronger lower bound also with respect to the 0/1 loss, i.e.~a lower bound on $\dc^\calD_{\eps}$ that is exponential (or even just super-polynomial) in both $\VCdim(\calH)$ and $\log\abs{\calX}$ ?  In particular, can we prove a $\poly(n)$ lower bound on $\dc^\calD_{\eps}$ for the class of all parities over $n$ bits, for which we do have a strong lower bound w.r.t. square loss?

In turns out that proving such a lower bounds for any explicit class $\calH$ will have significant complexity theoretic consequences. Suppose for example, we have an explicit class $\calH \subseteq \sbit^{\calX}$ for which we could prove, for some value of $\eps > 0$, that\\[-4mm]
$$
\dc_{\eps}^{\calD}(\calH) ~\ge~ (\log |\calH| |\calX|)^{\omega(1)} \cdot \frac{1}{\eps}\,.
$$
That is, we could establish a lower bound on $\dc_{\eps}^{\calD}(\calH)$ that is super-polynomial in $\log\abs{\calX}$ and in $\VCdim(\calH)$ (recall that $\VCdim(\calH)\leq\log\abs{\calH}$).  As shown by \cite{alman17probabilistic} (see \lemmaref{lem:alman-williams} \& \propositionref{prop:prob-dc-vs-pt}) it will follow that depth-$2$ threshold circuits computing $E_{\calH} : (h,x) \mapsto h(x)$ require size that is at least $(\log |\calH||\calX|)^{\omega(1)}$, for any binary encoding of $\calH$ and $\calX$.

Proving super-polynomial lower bounds on the size of depth-$2$ threshold circuits is a major frontier in Complexity Theory (the best lower bounds known so far is due to \cite{kane16super}, who show a lower bound of $\wtilde{\Omega}(n^{1.5})$ for an explicit $n$-bit function).  And so, establishing strong lower bounds on linear or kernel based learning with respect to the 0/1 loss for specific classes seems difficult.  This explains, perhaps, why recent work on the relative power of deep learning over kernel method focused on regression w.r.t.~the square loss, and indicates that establishing similar results also for classification might not be so easy.

Since proving explicit lower bounds for $\dc_{\eps}^{\calD}(\calH)$ faces a complexity theoretic barrier, we could ask for lower bounds on $\dc_{\eps}^{\calD,\lhinge}(\calH)$. Interestingly, it was shown by \cite{balcan08similarity} (stated in our notations) that $\mc_{\eps}^{\calD,\lhinge}(\calH) \ge (\frac{2}{\pi} - \eps) \cdot \Omega\inparen{\SQdim^{\calD}(\calH)^{\nicefrac{1}{2}}}$, which suggests the following open question.
\paragraph{Open Question.} Can we prove lower bounds on $\dc_{\eps}^{\calD,\lhinge}(\calH)$ in terms of $\SQdim^{\calD}(\calH)$?

\section{Summary\remove{ and Open problems}}\label{sec:discussions}

We formalized a notion of Linear Learning ($\Lin_{\eps}^{\ell}$) and Kernel Learning ($\Ker_{\eps}^{\ell}$) with respect to any loss $\ell$. We defined probabilistic variants of the classic notions of dimensional complexity ($\dc_{\eps}^{\ell}$) and margin complexity ($\mc_{\eps}^{\ell}$), which we show are equivalent to a notion of ``guaranteed'' Linear Learning ($\gLin_{\eps}^{\ell}$) and Kernel Learning ($\gKer_{\eps}^{\ell}$) respectively, where the guarantee follows from standard generalization bounds which follow from dimension-based or norm-based arguments respectively. For each of the notions above, we also defined a {\em distributional version}, where we fix a marginal distribution $\calD$ over the input space $\calX$.

We showed that $\dc_{\eps}^{\ell}$ and $\mc_{\eps}^{\ell}$ (resp. $\dc_{\eps}^{\calD,\ell}$ and $\mc_{\eps}^{\calD,\ell}$) are {\em sufficient} for learning with finite dimension or with finite norm embeddings (respectively in the distribution dependent setting). Morover, in the case of $\ell = \lzeroone$ loss, $\dc_{\eps}^{\lzeroone}$ can be exponentially smaller than the classic notion of $\dc^{\lzeroone}$. We also showed that the distributional versions $\dc_{\eps}^{\calD,\ell}$ and $\mc_{\eps}^{\calD,\ell}$ are upper bounded in terms of the VC-dimension.

Finally, we showed that $\dc_{\eps}^{\calD,\ell}$ is {\em necessary} for learning with either finite dimension or with finite norm embeddings, in the distribution dependent setting and hence also in the distribution independent setting. These connections are summarized in \figureref{fig:complexity-zoo}.

In the case of $\ell = \lsq$, we proved a lower bound $\dc_{\eps}^{\calD,\lsq}$ in terms of the notion of $\minEVdim^{\calD}$, which in turn is lower bounded by $\SQdim^{\calD}$; this allows us to re-prove (and even improve upon) similar lower bounds proved in literature \citep{yehudai19power,allenzhu19resnets,allenzhu20backward}. In the case of $\ell =\lzeroone$, we prove a lower bound on $\dc_{\eps}^{\calD,\lzeroone}$ of $\Omega(n)$ for the class of $1$-sparse predictors on $n$ variables. But this is only logarithmic in $|\calX|$. However, we identified a complexity theoretic barrier, namely that any lower bound on $\dc_{\eps}^{\calD,\lzeroone}$ for any $\calD$ that is super-polynomial in $\log(|\calH||\calX|)$ for any explicit class $\calH$ will imply super-polynomial lower bounds for depth-$2$ threshold circuits which is long-standing open question in circuit complexity.

We hope that our notions of probabilistic dimensional and margin complexity prove useful in the further understanding of the limitations of linear and kernel learning.

\acks{
We thank Josh Alman, Shai Ben-David, Avrim Blum, Brian Bullins, Surbhi Goel, Mika G\"o\"os, Suriya Gunasekar, Adam Klivans, Nati Linial, Raghu Meka, Prasad Raghavendra, Sasha Razborov, Ohad Shamir, Sasha Sherstov, Blake Woodworth and Gilad Yehudai for helpful discussions. We would especially like to thank Surbhi for suggesting the formulation in \corollaryref{cor:dcl2-lowerbound-SQ} in terms of SQ dimension and Mika for suggesting the proof of \theoremref{thm:dcl01-lowerbound-sparse}.

Research was partially supported by NSF BIGDATA award 1546500 and NSF IIS/RI award 1764032.  Part of the work was done when the authors were visiting the Simons Institute as part of the program on {\em Foundations of Deep Learning}.
}

\input{main.bbl}

\appendix

\section{\boldmath Relating $\dc$ and $\mc$ : Proof of Lemma~\ref{lem:dc-mc}} \label{apx:proof-dc-mc}

\begin{proofof}{\lemmaref{lem:dc-mc}}
For any Hilbert space $\bbH$, by the lemma of \citep{johnson1984extensions}, we have that there exists a distribution $\calA$ over projections $\pi:\bbH \to \bbR^{d}$ such that for any $u, v \in \bbH$,
\begin{equation}\label{eqn:jl}
\Prob_{\pi \sim \calA} \insquare{\inabs{\inangle{u,v}_{\bbH} - \inangle{\pi(u), \pi(v)}_{\bbR^k}} > \tau} < \delta \qquad \text{ for } d = \Theta\inparen{\frac{\|u\|^2_{\bbH}\|v\|^2_{\bbH}}{\tau^2} \log \frac{1}{\delta}}\,.
\end{equation}
We can also derive an expectation version of the above to get\\[-4mm]
\begin{equation}\label{eqn:avg-jl-2}
\Ex_{\pi \sim \calA} \inabs{\inangle{u,v}_{\bbH} ~-~ \inangle{\pi(u), \pi(v)}_{\bbR^k}}^2 ~\le~ \calO\inparen{\frac{\|u\|^2_{\bbH} \|v\|^2_{\bbH}}{d}}
\end{equation}
which also implies\\[-5mm]
\begin{equation}\label{eqn:avg-jl}
\Ex_{\pi \sim \calA} \inabs{\inangle{u,v}_{\bbH} ~-~ \inangle{\pi(u), \pi(v)}_{\bbR^k}} ~\le~ \calO\inparen{\frac{\|u\|_{\bbH} \|v\|_{\bbH}}{\sqrt{d}}}
\end{equation}
Let $\calP_{\rm mc}$ be a distribution over embeddings $\varphi: \calX \to \bbH$ with $\|\varphi\|_{\infty} \le 1$ that realizes the definition of $\mc_{\eps}^{\ell}(\calH) =: R$. That is, for all distributions $\calD$ over $\calX$ and all $h\in\calH$,
\[\Ex\limits_{\varphi \sim \calP_{\rm mc}} \insquare{ \inf_{w \in \calB(\bbH;R)} \err_{\calD,h}^{ \ell}(\inangle{w, \varphi(\cdot)})} \le \eps\,.\]

Consider a distribution $\calP_{\rm dc}$ over embeddings $\psi : \calX \to \bbR^d$ obtained as $\psi(x) = \pi(\varphi(x))$ for independently sampled $\varphi \sim \calP_{\rm mc}$ and $\pi \sim \calA$. For any distribution $\calD$ over $\calX$ and any $h\in\calH$, we have,
\begin{align}
\Ex_{\psi \sim \calP_{\rm dc}} \insquare{ \inf_{w \in \bbR^d} \err_{\calD,h}^{ \ell}(\inangle{w, \psi(\cdot)})}
&~\le~ \Ex_{\substack{\varphi \sim \calP_{\rm mc} \\ \pi \sim \calA}} \insquare{ \inf_{w \in \calB(\bbH;R)}\  \err_{\calD,h}^{ \ell}(\inangle{\pi(w), \pi(\varphi(\cdot))})} \nonumber\\
&~\le~ \Ex_{\varphi \sim \calP_{\rm mc}} \insquare{ \inf_{w \in \calB(\bbH;R)}\ \Ex_{\pi \sim \calA} \err_{\calD,h}^{ \ell}(\inangle{\pi(w), \pi(\varphi(\cdot))})}\label{eqn:jl-1}
\end{align}

\paragraph{Proof of (i).} We first infer from \eqref{eqn:jl} that for any $u, v \in \bbH$,
\begin{equation}\label{eqn:jl-revisit}
    \Ex_{\pi\sim\calA} \insquare{\ind\Set{\inangle{\pi(u),\pi(v)} < 0}} ~\le~ \ind\Set{\inangle{u,v} < \tau} + \delta \qquad \text{ for } d = \Theta\inparen{\frac{\|u\|^2_{\bbH}\|v\|^2_{\bbH}}{\tau^2} \log \frac{1}{\delta}}
\end{equation}
Starting from the inner term in \eqref{eqn:jl-1}, for any $w\in \bbH$ with $\|w\|_{\bbH} \le R$ and $\|\varphi\|_\infty \le 1$
\begin{align*}
\Ex_{\pi \sim \calA} \err_{\calD,h}^{ \lzeroone}(\inangle{\pi(w), \pi(\varphi(\cdot))})
&~=~ \Ex_{x\sim\calD} \Ex_{\pi \sim \calA} \ind\Set{\inangle{\pi(w), \pi(\varphi(x))} h(x) \le 0}\\
&~\le~ \Ex_{x\sim\calD} \ind\Set{\inangle{w,\varphi(x)}h(x) \le 1} + \eta \qquad \ldots \text{(from \eqref{eqn:jl-revisit})}\\
&~=~ \err_{\calD,h}^{ \lmargin}(\inangle{\pi(w), \pi(\varphi(\cdot))}) + \eta
\end{align*}
where we instantiate \eqref{eqn:jl-revisit} with $\tau = 1$, $\delta = \eta$, by setting $d = O(R^2 \log (1/\eta))$. Plugging this upper bound into \eqref{eqn:jl-1}, we get our desired goal
\begin{align*}
\Ex_{\psi \sim \calP_{\rm dc}} \insquare{ \inf_{w \in \bbR^d} \err_{\calD,h}^{ \lzeroone}(\inangle{w, \psi(\cdot)})}
&~\le~ \Ex_{\varphi \sim \calP_{\rm mc}} \insquare{ \inf_{w \in \calB(\bbH;R)} \err_{\calD,h}^{ \lmargin}(\inangle{w, \varphi(\cdot)})} + \eta ~\le~ \eps + \eta\,.
\end{align*}

\paragraph{Proof of (ii).} We use \eqref{eqn:avg-jl}. For any $w\in \bbH$ with $\|w\|_{\bbH} \le R$, we have from $L$-Lipschitzness of $\ell$ and $\|\varphi\|_{\infty} \le 1$ that\\[-6mm]
\begin{align*}
&\Ex_{\pi \sim \calA} \insquare{\err_{\calD,h}^{ \ell}(\inangle{\pi(w), \pi(\varphi(\cdot))})} - \err_{\calD,h}^{ \ell}(\inangle{w, \varphi(\cdot)})\\
&~=~ \Ex_{x\sim\calD}\ \Ex_{\pi\sim\calA} \insquare{\ell(\inangle{\pi(w),\pi(\varphi(x))},h(x)) - \ell(\inangle{w,\varphi(x)},h(x))}\\
&~\le~ L \cdot \Ex_{x \sim \calD}\ \Ex_{\pi\sim\calA} \inabs{\inangle{\pi(w), \pi(\varphi(x))} - \inangle{w, \varphi(x)})}\\
&~\le~ \calO\inparen{\frac{LR}{\sqrt{d}}}
\end{align*}
Combining this with \eqref{eqn:jl-1}, we get,
\begin{align*}
\Ex_{\psi \sim \calP_{\rm dc}} \insquare{ \inf_{w \in \bbR^d} \err_{\calD,h}^{ \ell}(\inangle{w, \psi(\cdot)})}
&~\le~ \Ex_{\varphi \sim \calP_{\rm mc}} \insquare{\inf_{w \in \calB(\bbH;R)} \err_{\calD,h}^{ \ell}(\inangle{w, \varphi(\cdot)}) + \calO\inparen{\frac{LR}{\sqrt{d}}}}\\
&~\le~ \eps + \calO\inparen{\frac{LR}{\sqrt{d}}}
\end{align*}
Thus, we get our desired statement for a choice of $d = \calO\inparen{LR/\eta}^2$.

\paragraph{Proof of (iii).} We use \eqref{eqn:avg-jl-2} and \eqref{eqn:avg-jl}. We use \eqref{eqn:avg-jl}. For any $w\in \bbH$ with $\|w\|_{\bbH} \le R$ we have 
\begin{align*}
&\Ex_{\pi \sim \calA} \insquare{\err_{\calD,h}^{\lsq}(\inangle{\pi(w), \pi(\varphi(\cdot))})} - \err_{\calD,h}^{\lsq}(\inangle{w, \varphi(\cdot)})\\
&~=~ \frac{1}{2} \ \  \Ex_{x\sim\calD}\ \Ex_{\pi\sim\calA} \insquare{(h(x) - \inangle{\pi(w),\pi(\varphi(x))})^2 - (h(x) - \inangle{w,\varphi(x)})^2}\\
&~\le~ \frac{1}{2} \ \  \Ex_{x\sim\calD}\ \Ex_{\pi\sim\calA} \inabs{h(x) - \inangle{w,\varphi(x)}} \cdot \inabs{\inangle{\pi(w),\pi(\varphi(x))} - \inangle{w,\varphi(x)}}\\
&\phantom{~\le~} + \frac{1}{2} \ \  \Ex_{x\sim\calD}\ \Ex_{\pi\sim\calA} \inabs{\inangle{\pi(w),\pi(\varphi(x))} - \inangle{w,\varphi(x)}}^2\\
&~\le~ \Ex_{x\sim\calD}\ \inabs{h(x) - \inangle{w,\varphi(x)}} \cdot \calO\inparen{\frac{R}{\sqrt{d}}} + \calO\inparen{\frac{R^2}{d}}\\
&~\le~ \inparen{\Ex_{x\sim\calD}\ \inabs{h(x) - \inangle{w,\varphi(x)}}^2}^{1/2} \cdot \calO\inparen{\frac{R}{\sqrt{d}}} + \calO\inparen{\frac{R^2}{d}}\\
&~=~ \err_{\calD,h}^{\lsq}(\inangle{w, \varphi(\cdot)})^{1/2} \cdot \calO\inparen{\frac{R}{\sqrt{d}}} + \calO\inparen{\frac{R^2}{d}}
\end{align*}
Combining this with \eqref{eqn:jl-1}, we get,
\begin{align*}
&\Ex_{\psi \sim \calP_{\rm dc}} \insquare{ \inf_{w \in \bbR^d} \err_{\calD,h}^{\lsq}(\inangle{w, \psi(\cdot)})}\\
&~\le~ \Ex_{\varphi \sim \calP_{\rm mc}} \insquare{\inf_{w \in \calB(\bbH;R)} \err_{\calD,h}^{\lsq}(\inangle{w, \varphi(\cdot)}) + \err_{\calD,h}^{\lsq}(\inangle{w, \varphi(\cdot)})^{1/2} \cdot \calO\inparen{\frac{R}{\sqrt{d}}} + \calO\inparen{\frac{R^2}{d}}}\\
&~\le~ \eps + \calO\inparen{\frac{\sqrt{\eps}R}{\sqrt{d}}} + \calO\inparen{\frac{R^2}{d}}
\end{align*}
where, we use that $\Ex_{\varphi} \inf_w \err_{\calD,h}^{\lsq}(\inangle{w,\varphi(\cdot)})^{1/2} \le \inparen{\Ex_{\varphi} \inf_w \err_{\calD,h}^{\lsq}(\inangle{w,\varphi(\cdot)})}^{1/2}\le \sqrt{\eps}$. Thus, we get our desired statement for a choice of $d = R^2 \cdot \calO\inparen{(\eps+\eta)/\eta^2}$. This completes the proof for all the parts (i), (ii) and (iii). The analogous cases relating $\dc_{\eps+\eta}^{\calD,\ell}$ and $\mc_{\eps}^{\calD,\ell}$ follows similarly.
\end{proofof}

\section{Proofs of Separation between Deterministic and Probabilistic Dimension Complexity}

\subsection{Exponential gap : Proof of Theorem~\ref{thm:dc-vs-probdc-sep}}\label{apx:proof-dc-vs-probdc-sep}

\noindent We first introduce a variant of probabilistic dimension complexity that requires a stronger point-wise notion of correctness.

\begin{definition} \label{def:prob-dc-pt}
Fix a hypothesis class $\calH \subseteq \calY^{\calX}$ and a loss $\ell$ and a parameter $\eps \ge 0$. The {\em point-wise probabilistic dimension complexity} $\dc_{\eps}^{\pt,\ell}(\calH)$ is the smallest $d$ for which there exists a distribution $\calP$ over a pair of embeddings $(\varphi:\calX \to \bbR^d, w:\calH \to \bbR^d)$ such that,
$$
\sup_{(x,h) \in \calX \times \calH} \ 
\Ex\limits_{(\varphi,w) \sim \calP} \  \insquare{\ell(\inangle{w(h), \varphi(x)}, h(x))} \le \eps\,.
$$
\end{definition}

\noindent This notion of point-wise probabilistic dimension complexity requires that (the distribution over) $w$ is chosen without the knowledge of the distribution $\calD$ over $\calX$ and hence is stronger than probabilistic dimension complexity as in~\definitionref{def:prob-dc}. In particular, we have the following.

\begin{proposition} \label{prop:prob-dc-vs-pt}
For all $\calH \subseteq \calY^{\calX}$, loss $\ell$ and parameter $\eps > 0$, it holds that,
$$
\sup_{\calD} \dc_{\eps}^{\calD,\ell}(\calH) ~\le~ \dc_{\eps}^{\ell}(\calH) ~\le~ \dc_{\eps}^{\pt,\ell}(\calH)
$$
\end{proposition}

\noindent The notion of $\dc_{\eps}^{\pt,\lzeroone}(\calH)$ is equivalent to the notion of {\em probabilistic sign-rank} studied in the communication complexity. In particular, stating in our notations, \cite{alman17probabilistic} showed that if the function $E_{\calH} : \calH \times \calX \to \sbit$ given by $E_{\calH}(h,x) := h(x)$ is computable by small depth-$2$ threshold circuits (for any encoding of $\calH$ and $\calX$ into bits), then $\dc_{\eps}^{\pt,\lzeroone}(\calH)$ is also small.

\begin{lemma}[\cite{alman17probabilistic}] \label{lem:alman-williams}
If $E_{\calH}$ is computable by a depth-$2$ threshold circuit of size $s$, then
$$
\dc_{\eps}^{\pt,\lzeroone}(\calH) \le O\inparen{\frac{s^2 \log^2(|\calH|\cdot|\calX|)}{\eps}}
$$
\end{lemma}

\noindent \theoremref{thm:dc-vs-probdc-sep} now follows readily from a recent lower bound on sign-rank shown by \cite{chattopadhyay18short} for matrices that are computable by small depth-$2$ threshold circuits.\\

\begin{proofof}{\theoremref{thm:dc-vs-probdc-sep}}
We describe the construction of the class $\calH$, which is indexed by $\sbit^n$. To describe how an $h\in \calH$ {\em acts on} an $x \in \calX$, we divide the $n$ bits in $h$ and $x$ into $k$ blocks by writing $h = (h_1, \ldots, h_k)$ and $x = (x_1, \ldots, x_k)$ where each $h_i, x_i \in \sbit^{p}$ with $kp = n$. The hypothesis $h$ on input $x$ outputs $-1$ iff the largest index $i \in [k]$ for which $h_i = x_i$ holds is an odd index.\footnote{In communication complexity parlance, the associated $M_{\calH}$ would be called a ``pattern matrix''.} For $p = k^{1/3} + \log k$, it was shown by \cite{chattopadhyay18short} that
$$\dc(\calH) \ge 2^{\Omega(n^{1/4})}\,.$$
\cite{chattopadhyay18short} also observe that $E_{\calH} : (h,x) \mapsto h(x)$ is computable by a depth-2 threshold circuit of size $O(n)$. Thus, from \lemmaref{lem:alman-williams}, we have that
$$
\dc_{\eps}^{\pt,\lzeroone}(\calH) ~\le~ O\inparen{\frac{n^4}{\eps}}
$$
Combining with \propositionref{prop:prob-dc-vs-pt} we get our desired separation.
\end{proofof}

\subsection{``Infinite'' gap : Proof of Theorem~\ref{thm:dc-vs-prob-dist-dc-sep}}\label{apx:proof-dc-vs-prob-dist-dc-sep}

We first prove \lemmaref{lem:dc_vc_upper} that probabilistic distributional dimension complexity can be upper bounded in terms of VC dimension.\\[-3mm]

\begin{proofof}{\lemmaref{lem:dc_vc_upper}}
A classic result due to \cite{haussler1995sphere} shows that for any distribution $\calD$ over $\calX$ there exists a cover $\calC_{\eps} \subseteq \calH$, with $|\calC_\eps|\leq c \cdot \vc(\calH) \cdot (K/\eps)^{\vc(\calH)}$ for some universal constants $c,K$, such that,
$$\forall h\in\calH, \ \exists c_h\in\calC_{\eps} \text{ such that } \Prob_{x\sim\calD}[h(x)\neq c_{h}(x)]\leq \eps\,.$$
\noindent Thus for any given distribution $\calD$, we can construct a (deterministic) embedding $\varphi : \calX \to \bbR^{|\calC_\eps|}$ given as $\varphi(x)=(c(x))_{c\in\calC_\eps}$ and $w(h)=(\ind[c = c_h])_{c\in\calC_\eps}$ satisfying the property that, 
$$\forall h\in\calH \ : \ \Ex_{x\sim\calD} \ind[h(x)\neq \sign(\inangle{\varphi(x),w(h)})] \leq \eps.$$
This implies that $\dc_{\eps}^{\calD}(\calH)\leq |\calC_{\eps}|$. Note that, since $\inangle{w(h), \varphi(x)}$ always takes values in $\sbit$, $\dc_{2\eps}^{\calD,\lsq}(\calH)$ and $\dc_{2\eps}^{\calD,\lhinge}(\calH)$ are also at most $|\calC_{\eps}|$.

Also, observe that if we can scale $\varphi$ by $1/\sqrt{|\calC_{\eps}|}$, we will have $\|\varphi\|_{\infty} \le 1$. To compensate for this, we can scale up $w$ by $\sqrt{|\calC_{\eps}|}$ and get the desired upper bound on $\mc_{\eps}^{\calD,\ell}(\calH)$.
\end{proofof}

\begin{proofof}{\theoremref{thm:dc-vs-prob-dist-dc-sep}}
\cite{alon2016sign} showed that for $\calX=\sbit^n$ there exists a hypothesis class $\calH\subseteq \sbit^\calX$ such that $\VCdim(\calH)=2$ but $\dc(\calH) \ge 2^{\Omega(n)}$. Note that $\dc^{\lsq}(\calH)$ and $\dc^{\lhinge}(\calH)$ are each larger than $\dc(\calH)$. Also note that $\mc(\calH) \ge \Omega(\sqrt{\dc(\calH) / n})$ (from the classic result relating $\mc$ and $\dc$). Thus we get the desired lower bound on $\mc(\calH)$, $\mc^{\lsq}(\calH)$ and $\mc^{\lhinge}(\calH)$ as well.
On the other hand, from \lemmaref{lem:dc_vc_upper}, we get that both $\dc_{\eps}^{\calD,\ell}(\calH)$ (for $\ell \in \Set{\lzeroone,\lsq,\lhinge}$) and $\mc_{\eps}^{\calD,\ell}(\calH)$ (for $\ell \in \Set{\lmargin,\lsq,\lhinge}$) are at most $\calO\inparen{1/\eps^{2}}$ for every distribution $\calD$ over $\calX$.
\end{proofof}

\section{Proofs of Upper and Lower Bounds on Learning}\label{apx:proof-learn-dc-mc-upper}

\subsection{Learning via Random embeddings : Proof of Theorems~\ref{thm:lin-tilde-vs-dc} and \ref{thm:ker-tilde-vs-mc}}

\begin{proofof}{\theoremref{thm:lin-tilde-vs-dc}} 
$\Lin_{\eps}^{\ell}(\calH) ~\le~ \gLin_{\eps}^{\ell}(\calH)$ and $\Omega\inparen{\frac{\dc_{\eps}^{\ell}(\calH)}{\eps^2}} \le \gLin_{\eps}^{\ell}(\calH) \le \calO\inparen{\frac{\dc_{\eps/2}^{\ell}(\calH)}{\eps^2}}$


\noindent Let $\calP$ be the distribution over embeddings $\varphi : \calX \to \bbR^d$ underlying the definition of $\gLin_{\eps}^{\ell}(\calH) =: m$. That is, we have for any realizable distribution $\scrD$ over $\calX \times \calY$ that
\begin{equation}\label{eqn:lin-proof-1}
\Ex_{\varphi \sim \calP} \Ex_{S \sim \scrD^{m}} \insquare{\inf_{w \in \bbR^d} \err_S^{\ell}(\inangle{w,\varphi(\cdot)})} + C_{\dc}^{\ell} \cdot \sqrt{\frac{d}{m}} ~\le~ \eps\,.
\end{equation}
On the other hand, from standard generalization bounds (cf. \equationref{eqn:dc-gen}), we have for any choice of $\varphi : \calX \to \bbR^d$ and $\scrD$ that
$$
\Ex_{S \sim \scrD^m} \insquare{\sup_{w \in \bbR^d} \err_{\scrD}^{\ell}(\inangle{w,\varphi(\cdot)}) - \err_S^{\ell}(\inangle{w,\varphi(\cdot)})} ~\le~ C_{\dc}^{\ell} \cdot\sqrt{\frac{d}{m}}\,.
$$
And hence,
\begin{align*}
\Ex_{S \sim \scrD^{m}} \insquare{\sup_{w \in \ERM_{\varphi}^{\ell}(S)} \err_{\scrD}^{\ell}(\inangle{w,\varphi(\cdot)})}
&~\le~  \Ex_{S \sim \scrD^{m}} \insquare{\inf_{w \in \bbR^d} \err_S^{\ell}(\inangle{w,\varphi(\cdot)})} + C_{\dc}^{\ell} \sqrt{\frac{d}{m}}
\end{align*}
Thus, taking expectation over $\varphi\sim\calP$, we have from \eqref{eqn:lin-proof-1} that
$$
\Ex_{\varphi \sim \calP} \Ex_{S \sim \scrD^{m}} \insquare{\sup_{w \in \ERM_{\varphi}^{\ell}(S)} \err_{\scrD}^{\ell}(\inangle{w,\varphi(\cdot)})} ~\le~ \eps
$$
Thus, we get $\Lin_{\eps}^{\ell}(\calH) \le m = \gLin_{\eps}^{\ell}(\calH)$. It also follows that $\dc_{\eps}(\calH) \le \eps^2 \gLin_{\eps}^{\ell}(\calH)$, since firstly $d \le \eps^2 m$ by definition of $\gLin_{\eps}^{\ell}(\calH) = m$. Moreover, if we let $\scrD$ to be the distribution sampled as $x \sim \calD$ and $y = h(x)$ for some $h \in \calH$, we get,
$$
\Ex_{\varphi \sim \calP} \insquare{\inf_{w \in \bbR^d} \err_{\calD,h}^{\ell}(\inangle{w,\varphi(\cdot)})} ~\le~
\Ex_{\varphi \sim \calP} \Ex_{S \sim \scrD^{m}} \insquare{\sup_{w \in \ERM_{\varphi}^{\ell}(S)} \err_{\scrD}^{\ell}(\inangle{w,\varphi(\cdot)})} ~\le~ \eps
$$

\noindent Finally, it remains to show that $\gLin_{\eps}^{\ell}(\calH) \le O(\dc_{\eps/2}^{\ell}(\calH) / \eps^2)$. Let $\calP$ be the distribution over embeddings $\varphi : \calX \to \bbR^d$ that realizes the definition of $\dc_{\eps/2}^{\ell}(\calH) =: d$. Thus, we have for any realizable distribution $\scrD$ over $\calX \times \calY$ that
\begin{equation}\label{eqn:lin-dc-1}
\Ex_{\varphi \sim \calP} \insquare{\inf_{w \in \bbR^d} \err_{\scrD}^{\ell}(\inangle{w,\varphi(\cdot)})} ~\le~ \frac{\eps}{2}\,.
\end{equation}

\noindent Now, for any choice of $\varphi : \calX \to \bbR^d$ and any $w_* \in \bbR^d$ we have
$$
\Ex_{S \sim \scrD^{m}} \insquare{\inf_{w \in \bbR^d} \err_S^{\ell}(\inangle{w,\varphi(\cdot)})}
~\le~ \Ex_{S \sim \scrD^{m}} \insquare{ \err_S^{\ell}(\inangle{w_*,\varphi(\cdot)})} ~=~ \err_{\scrD}^{\ell}(\inangle{w_*,\varphi(\cdot)})
$$
Taking infimum over $w_*$ (in RHS) and an expectation over $\varphi \sim \calP$, we get,
\begin{align*}
&\Ex_{\varphi \sim \calP} \Ex_{S \sim \scrD^{m}} \insquare{\inf_{w \in \bbR^d} \err_S^{\ell}(\inangle{w,\varphi(\cdot)})} + C_{\dc}^{\ell} \cdot \sqrt{\frac{d}{m}}\\
&~\le~ \Ex_{\varphi\sim\calP} \insquare{\inf_{w \in \bbR^d} \err_{\scrD}^{\ell}(\inangle{w,\varphi(\cdot)})} + C_{\dc}^{\ell} \cdot \sqrt{\frac{d}{m}}\\
&~\le~ \frac{\eps}{2} + C_{\dc}^{\ell} \cdot \sqrt{\frac{d}{m}} \qquad \ldots (\text{from \eqref{eqn:lin-dc-1}})\\
&~\le~ \eps \qquad \ \ \ldots (\text{for a choice of } m = \calO(d/\eps^2))
\end{align*}
This establishes $\gLin_{\eps}^{\ell}(\calH) \le \calO(\dc_{\eps/2}^{\ell}(\calH) / \eps^2)$, thereby completing the proof for the distribution-independent case.
The distribution-dependent analogs follow in an identical manner.
\end{proofof}

\begin{proofof}{\theoremref{thm:ker-tilde-vs-mc}} 
$\Ker_{\eps}^{\ell}(\calH) ~\le~ \gKer_{\eps}^{\ell}(\calH)$ and
$\Omega\inparen{\frac{\mc_{\eps}^{\ell}(\calH)^2}{\eps^2}}
~\le~ \gKer_{\eps}^{\ell}(\calH) ~\le~
\calO\inparen{\frac{\mc_{\eps/2}^{\ell}(\calH)^2}{\eps^2}}$\\

\noindent This proof is very similar to that of \theoremref{thm:lin-tilde-vs-dc}, except that we use norm-based generalization bounds instead of dimension-based ones. We present the proof for $\lzeroone$/$\lmargin$ and the case of general Lipshitz $\ell$ follows in a similar manner.

Let $\calP$ be the distribution over embeddings $\varphi : \calX \to \bbH$ underlying the definition of $\gKer_{\eps}(\calH) =: m$. That is, we have for any realizable distribution $\scrD$ over $\calX \times \calY$ that
\begin{equation}\label{eqn:ker-proof-1}
\Ex_{\varphi \sim \calP} \Ex_{S \sim \scrD^{m}} \insquare{\inf_{w \in \calB(\bbH;R)} \err_S^{\lmargin}(\inangle{w,\varphi(\cdot)})} + C_{\mc} \cdot \frac{R}{\sqrt{m}} ~\le~ \eps\,.
\end{equation}
On the other hand, from standard norm based generalization bounds (see \equationref{eq:norm-based-uniform}), we have for any choice of $\varphi : \calX \to \bbH$ and $\scrD$ that
$$
\Ex_{S \sim \scrD^m} \insquare{\sup_{w \in \calB(\bbH;R)} \err_{\scrD}^{\lzeroone}(\inangle{w,\varphi(\cdot)}) - \err_S^{\lmargin}(\inangle{w,\varphi(\cdot)})} ~\le~ C_{\mc} \cdot\frac{R}{\sqrt{m}}\,.
$$
And hence,
\begin{align*}
\Ex_{S \sim \scrD^{m}} \insquare{\sup_{w \in \ERM_{\varphi}^{\lmargin}(S;R)} \err_{\scrD}^{\lzeroone}(\inangle{w,\varphi(\cdot)})}
&~\le~  \Ex_{S \sim \scrD^{m}} \insquare{\inf_{w \in \calB(\bbH;R)} \err_S^{\lmargin}(\inangle{w,\varphi(\cdot)})} + C_{\mc} \frac{R}{\sqrt{m}}
\end{align*}
Thus, taking expectation over $\varphi\sim\calP$, we have from \eqref{eqn:ker-proof-1} that
$$
\Ex_{\varphi \sim \calP} \Ex_{S \sim \scrD^{m}} \insquare{\sup_{w \in \ERM_{\varphi}^{\lzeroone}(S)} \err_{\scrD}^{\lzeroone}(\inangle{w,\varphi(\cdot)})} ~\le~ \eps
$$
Thus, we get $\Ker_{\eps}(\calH) \le m = \gKer_{\eps}(\calH)$. It also follows that $\mc_{\eps}(\calH) \le \eps \sqrt{\gLin_{\eps}(\calH)}$, since firstly $R \le \eps \sqrt{m}$ by definition of $\gKer_{\eps}(\calH) = m$. Moreover, if we let $\scrD$ to be the distribution sampled as $x \sim \calD$ and $y = h(x)$ for some $h \in \calH$, we get,
$$
\Ex_{\varphi \sim \calP} \insquare{\inf_{w \in \calB(\bbH;R)} \err_{\calD,h}^{\lzeroone}(\inangle{w,\varphi(\cdot)})} ~\le~
\Ex_{\varphi \sim \calP} \Ex_{S \sim \scrD^{m}} \insquare{\sup_{w \in \ERM_{\varphi}^{\lzeroone}(S;R)} \err_{\scrD}^{\lzeroone}(\inangle{w,\varphi(\cdot)})} ~\le~ \eps
$$

\noindent Finally, it remains to show that $\gKer_{\eps}(\calH) \le O(\mc_{\eps/2}(\calH) / \eps^2)$. Let $\calP$ be the distribution over embeddings $\varphi : \calX \to \bbH$ with $\|\varphi\|_{\infty} \le 1$ that realizes the definition of $\mc_{\eps/2}(\calH) =: R$. Thus, we have for any realizable distribution $\scrD$ over $\calX$ that
\begin{equation}\label{eqn:ker-mc-1}
\Ex_{\varphi \sim \calP} \insquare{\inf_{w \in \calB(\bbH;R)} \err_{\scrD}^{\lmargin}(\inangle{w,\varphi(\cdot)})} ~\le~ \frac{\eps}{2}\,.
\end{equation}

\noindent Now, for any choice of $\varphi : \calX \to \bbH$ and any $w_* \in \bbH$ with $\|w_*\|_{\bbH} \le R$ we have
\begin{align*}
\Ex_{S \sim \scrD^{m}} \insquare{\inf_{w \in \calB(\bbH;R)} \err_S^{\lmargin}(\inangle{w,\varphi(\cdot)})}
&~\le~ \Ex_{S \sim \scrD^{m}} \insquare{ \err_S^{\lmargin}(\inangle{w_*,\varphi(\cdot)})}\\
&~=~ \err_{\scrD}^{\lmargin}(\inangle{w_*,\varphi(\cdot)})
\end{align*}
Finally, taking expectation over $\varphi \sim \calP$ and taking infimum over $w_*$ (in RHS), we get
\begin{align*}
& \Ex_{\varphi \sim \calP} \Ex_{S \sim \scrD^{m}} \insquare{\inf_{w \in \calB(\bbH;R)} \err_S^{\lmargin}(\inangle{w,\varphi(\cdot)})} + C_{\mc} \cdot \frac{R}{\sqrt{m}}\\
&~\le~ \Ex_{\varphi\sim\calP} \insquare{\inf_{w \in \calB(\bbH;R)} \err_{\scrD}^{\lmargin}(\inangle{w,\varphi(\cdot)})} + C_{\mc} \cdot \frac{R}{\sqrt{m}}\\
&~\le~ \frac{\eps}{2} + C_{\mc} \cdot \frac{R}{\sqrt{m}} \ \ \ldots \text{(from \eqref{eqn:ker-mc-1})}\\
&~\leq~ \eps \qquad \ldots (\text{for a choice of $m = \calO(R^2/\eps^2)$})
\end{align*}
This establishes $\gKer_{\eps}(\calH) \le \calO(\mc_{\eps/2}(\calH) / \eps^2)$, thereby completing the proof for the distribution-independent case. The distribution-dependent analogs follow in an identical manner.
\end{proofof}

\subsection{Lower Bound on Learning : Proof of Theorem~\ref{thm:learn-lowerbound}}

\begin{proofof}{\theoremref{thm:learn-lowerbound}}
We start with part (i). The first inequality of $\Lin_{\eps}^{\ell}(\calH) \ge \Lin_{\eps}^{\calD,\ell}(\calH)$ holds by definition; we focus on the second inequality.
Let $\calD$ be an arbitrary distribution over $\calX$ and $\eps>0$. Let $\calP$ be the distribution over embeddings $\varphi : \calX \to \bbR^d$ that realizes the definition of $\Lin_{\eps}(\calH) =: m$ for some $d$. For any $h \in \calH$, let $\scrD_h$ be the distribution over $\calX \times \calY$ given by $(x, h(x))$ for $x \sim \calD$ (that is, $\scrD_h$ is a distribution {\em realizable} under $\calH$). Thus, we have for any $h \in \calH$ that
$$
\Ex_{\varphi\sim\calP}\ \Ex_{S\sim\scrD_h^m}\ \insquare{ \inf_{w \in \bbR^d} \ \err_{\calD,h}^{ \ell}(\inangle{w, \varphi(\cdot)})} ~\leq~ \eps.
$$
For any $\varphi : \calX \to \bbR^d$ and $S\sim\scrD_h^m$, define the subspace spanned by embedding of the data $U_{\varphi, S} :={\rm span}\Set{\varphi(x_1),\ldots,\varphi(x_m)}$. We show that $\ERM_\varphi^{\ell}(S) \cap U_{\varphi, S} \ne \emptyset$; also known as ``Representer Theorem''. Namely, for any $w\in \ERM_{\varphi}^{\ell}(S)$, we can decompose $w = w^{||}+w^{\bot}$ such that $w^{||}\in U_{\varphi, S}$ and $\inangle{w^{\bot}, u}=0$ for all $u\in U_{\varphi, S}$. Thus, $\inangle{w,\varphi(x)}=\inangle{w^{||},\varphi(x)}$ for each $x \in S$. Hence $w^{||}\in \ERM_{\varphi}^{\ell}(S) \cap U_{\varphi, S}$. Thus, we have
$$
\Ex_{\substack{\varphi\sim\calP \\ S\sim\scrD_h^m}}\ \insquare{ \inf_{w\in U_{\varphi,S}} \ \err_{\calD,h}^{ \ell}(\inangle{w, \varphi(\cdot)})}
~\le~ \Ex_{\substack{\varphi\sim\calP \\ S\sim\scrD_h^m}} \ \insquare{ \inf_{w\in \ERM_{\varphi}^{\ell}(S) \cap U_{\varphi,S}} \ \err_{\calD,h}^{ \ell}(\inangle{w, \varphi(\cdot)})}  ~\leq~ \eps.
$$
Note that in the definition of $U_{\varphi,S}$, the labels sampled from $\scrD_h$ are unused. So we abuse notations and define $U_{\varphi,S}$ even for $S \sim \calD^m$.
In order to show that $\dc_{\eps}^{\calD,\ell}(\calH) \le m$ we construct a distribution $\calP_{\rm dc}$ over embeddings $\psi:\calX \to \bbR^{m}$ as follows: Sample $\varphi \sim \calP$ and $S \sim \calD^m$ and let $\psi(x) := \pi_{\varphi,S}(\varphi(x))$, where $\pi_{\varphi,S} : \bbR^d \to \bbR^m$ is the projection onto the subspace $U_{\varphi,S}$, expressed in terms of some canonical orthonormal basis. Note that for any $\varphi$, $S$ and $w \in U_{\varphi,S}$, it holds that $\inangle{w, \varphi(x)} = \inangle{\pi_{\varphi,S}(w), \psi(x)}$.
Thus, we get
$$
\Ex_{\psi \sim \calP_{\rm dc}} \insquare{\inf_{w \in \bbR^m} \err_{\calD,h}^{ \ell}\inparen{\inangle{w, \psi(\cdot)}}} =
\Ex_{\varphi \sim \calP} \ \Ex\limits_{S\sim\calD^m} \insquare{\inf_{w \in U_{\varphi,S}} \err_{\calD,h}^{ \ell} \inparen{\inangle{w, \varphi(\cdot)}}} \leq \eps.
$$

\noindent Part (ii) follows in an identical manner, so we skip the details.
\end{proofof}

\section{Proofs of Lower Bounds on Probabilistic Distributional Dimension Complexity}

\subsection{Case of square loss : Proof of Theorem~\ref{thm:dcl2-lowerbound-EV}} \label{apx:proof-dcl2-lowerbound-EV}

Our proof is inspired by the technique for lower bounding the approximate rank of a matrix due to \cite{alon13approxrank}.\\

\begin{proofof}{\theoremref{thm:dcl2-lowerbound-EV}}
For $\lambda > 2\eps$, let $t := \minEVdim^{\calD}(\calH; \lambda)$. That is, we have hypotheses $\calH_t = \Set{h_1, \ldots, h_t}$ with $\lambda_{\min}(G_{\calH_t}^{\calD}) \ge \lambda$. Let $d := \dc_{\eps}^{\calD,\lsq}(\calH)$, that is, there exists a distribution $\calP$ over pairs of embeddings $(\varphi : \calX \to \bbR^d, w : \calH \to \bbR^d)$\footnote{by choosing $w : h \mapsto \arg\inf_{w \in \bbR^d} \err_{\calD,h}^{\lsq}(\inangle{w,\varphi(\cdot)})$} such that for all $h \in \calH$,
$$
\Ex\limits_{(\varphi,w) \sim \calP} \insquare{\err_{\calD,h}^{ \lsq}(\inangle{w(h), \varphi(\cdot)})} \le \eps\,.
$$
In particular, if we average over $h \in \calH_t$,
$$
\Ex\limits_{(\varphi,w) \sim \calP} \ \Ex_{\substack{h \sim \calH_t\\ x \sim \calD}} \ \lsq(\inangle{w(h), \varphi(x)}, h(x)) \le \eps\,.
$$
Thus, we can fix a deterministic pair of embeddings $(\varphi_* : \calX \to \bbR^d, w_* : \calH \to \bbR^d)$ in the support of $\calP$ for which,
\begin{equation}\label{eqn:dcl2-1}
    \Ex_{\substack{h \sim \calH_t\\ x \sim \calD}} \ \lsq(\inangle{w_*(h), \varphi_*(x)}, h(x)) \le \eps\,.
\end{equation}
We have $G := G_{\calH_t}^{\calD} = MM^{\top}$ where $M \in \bbR^{t \times \calX}$ is given by $M(h,x) := \sqrt{\calD(x)} \cdot h(x)$ for all $h \in \calH_t$ and $x \in \calX$. Since $\lambda_{\min}(G) \ge \lambda$ we have for all $v \in \bbR^t$ that $v^{\top} G v \ge \lambda \|v\|_2^2$. In particular, we have
\begin{equation}\label{eqn:M-sing-lb}
    \forall v \in \bbR^t \quad : \quad \|M^{\top} v\|_2 \ge \sqrt{\lambda} \|v\|_2\,.
\end{equation}
On the other hand, the embedding pair $(\varphi_*, w_*)$ defines a rank-$d$ matrix $A \in \bbR^{t \times \calX}$ given by $A(h,x) := \sqrt{\calD(x)} \inangle{w_*(h), \varphi_*(x)}$ for each $h \in \calH_t$ and $x \in \calX$.

We define $E \in \bbR^{t \times \calX}$ as $E(h,x) := M(h,x) - A(h,x)$. We have from \eqref{eqn:dcl2-1}
$$
\|E\|_F^2 ~=~ \sum_{h \in \calH_t} \Ex_{x \sim \calD} \ (\inangle{w_*(h), \varphi_*(x)} - h(x))^2 ~\le~ 2\eps t
$$
In particular, we get
\begin{equation}\label{eqn:E-sing-ub}
    \sum_{i=1}^t \sigma_i(E)^2 \le 2\eps t\,.
\end{equation}

\noindent On the other hand, since $\rank(A) \le d$, 
there exists a subspace $S \subseteq \bbR^t$ of dimension $t - d$, such that $\|A^{\top} v\|_2 = 0$ for all $v \in S$. By triangle inequality, we get $0 = \|A^{\top} v\|_2 \ge \|M^{\top} v\|_2 - \|E^{\top} v\|_2$. From \eqref{eqn:M-sing-lb} we have $\|M^{\top}v\|_2 \ge \sqrt{\lambda}$. Thus, $\|E^{\top} v\|_2 \ge \sqrt{\lambda}$ for all $v \in S$. From the Courant-Fischer-Weyl min-max theorem, we get $\sigma_t(E) \ge \ldots \ge \sigma_{d+1}(E) \ge \sqrt{\lambda}$. Combining this with \eqref{eqn:E-sing-ub} implies $(t-d) \lambda \le 2\eps t$. Finally this implies $\dc_{\eps}^{\calD,\lsq}(\calH) \ge \dc_{\eps}^{\calD,\lsq}(\calH_t) \ge \inparen{1 - \frac{2\eps}{\lambda}} t$ as desired.
\end{proofof}

\subsection{Case of 0-1 loss : Proof of Theorem~\ref{thm:dcl01-lowerbound-sparse}} \label{apx:proof-dcl01-lowerbound-sparse}
In order to prove \theoremref{thm:dcl01-lowerbound-sparse}, we use a key fact from \cite{srebro04generalization} that provides an upper bound on the number of sign-matrices with sign-rank below a given bound. Namely, let $\mathsf{SM}(n,d)$ be the number of sign-matrices $M \in \sbit^{n \times n}$ with $\signrank(M) \le d$.

\begin{lemma}[\cite{srebro04generalization}]\label{lem:SM-upper-bound}
For all $n \ge k \ge 1$, it holds that $\mathsf{SM}(n,d) \le \inparen{\frac{8en}{d}}^{2dn}$.
\end{lemma}

\begin{proofof}{\theoremref{thm:dcl01-lowerbound-sparse}}
Let $\calP$ be the distribution over pair of embeddings $(\varphi : \calX_n \to \bbR^d, w : \calH_n^{\mathrm{1\text{-}sp}} \to \bbR^d)$\footnote{by choosing $w : h \mapsto \arg\inf_{w \in \bbR^d} \err_{\calD,h}^{\lzeroone}(\inangle{w,\varphi(\cdot)})$} that realizes the definition of $\dc_{\eps}^{\calD}(\calH_n^{\mathrm{1\text{-}sp}}) =: d$. If we sample $h$ uniformly in $\calH_n^{\mathrm{1\text{-}sp}}$, we have
\begin{equation}
\Prob_{\substack{x \sim\calD\\ h \sim \calH_n^{\mathrm{1\text{-}sp}}}} \insquare{\sign(\inangle{w(h), \varphi(x)}) \ne h(x)} \le \eps\,.\label{eqn:error-upper}
\end{equation}
On the other hand, consider a random subset $S \subseteq \calX_n$ of size $|S| = n$ and the hypothesis class $\calH_n^{\mathrm{1\text{-}sp}}$ evaluated only on inputs $x \in S$. A key step in this proof is to show that for $\gamma < 1/2$ and $c := d/n$,
\begin{equation}
\Prob_{S} \insquare{\Prob_{\substack{x \sim S\\h \sim \calH_n^{\mathrm{1\text{-}sp}}}} [\sign(\inangle{w(h), \varphi(x)}) \ne h(x)] \le \gamma} ~\le~ 2^{-n^2\inparen{1 - h(\gamma) - 2c \log\inparen{\frac{8e}{c}} - o(1)}}\,.\label{eqn:error-lower}
\end{equation}
This follows by a simple counting argument. For any $n \times n$ sign-matrix $M$ and $\gamma < 1/2$, the number of sign-matrices $A$ such that $\Prob_{(i,j) \sim [n] \times [n]} [M(i,j) \ne A(i,j)] \le \gamma$ is at most $\sum_{r=0}^{\gamma n^2} \binom{n^2}{r} \le 2^{(h(\gamma) + o(1))n^2}$.  From \lemmaref{lem:SM-upper-bound}, we have that $\mathsf{SM}(n,d) \le \inparen{\frac{8en}{d}}^{2dn} = 2^{2c \log \inparen{\frac{8e}{c}} n^2}$ where $c := d/n$. Thus, the number of $n \times n$ sign-matrices that agree with some sign-matrix of sign-rank $\le d$ on at least $(1-\gamma)$ fraction of the entries is at most $2^{\inparen{h(\gamma)+2c \log \inparen{\frac{8e}{c}}+o(1)} n^2}$.

On the other hand, the number of distinct $n \times n$ sign-matrices obtainable by sampling $S$ is at least $(2^n - n)^n \ge 2^{(1-o(1))n^2}$. Thus, \eqref{eqn:error-lower} follows.

By linearity of expectation, if we partition $\calX_n$ into subsets $S_1, \ldots, S_{2^n/n}$ each of size $n$, then in expectation, the fraction of $S_i$'s for which
$$\Prob_{\substack{x \sim S_i\\h \sim \calH_n^{\mathrm{1\text{-}sp}}}} [\sign(\inangle{w(h), \varphi(x)}) \ne h(x)] > \gamma$$
holds is at least $1 - 2^{-n^2\inparen{1 - h(\gamma) - 2c \log\inparen{\frac{8e}{c}} - o(1)}}$. In particular, we can fix such a partition for which this happens. And for such a partition, we get that,
\begin{align*}
\Prob_{\substack{x \sim\calD\\ h \sim \calH_n^{\mathrm{1\text{-}sp}}}} \insquare{\sign(\inangle{w(h), \varphi(x)}) \ne h(x)}
&~=~ \Prob_{i} \Prob_{\substack{x \sim S_i\\ h \sim \calH_n^{\mathrm{1\text{-}sp}}}} \insquare{\sign(\inangle{w(h), \varphi(x)}) \ne h(x)}\\
&~>~ \gamma \cdot \inparen{1 - 2^{-n^2\inparen{1 - h(\gamma) - 2c \log\inparen{\frac{8e}{c}} - o(1)}}}\,.
\end{align*}
Combining this with \eqref{eqn:error-upper}, we get for any choice of $\gamma$ that
$$
\gamma \cdot \inparen{1 - 2^{-n^2\inparen{1 - h(\gamma) - 2c \log\inparen{\frac{8e}{c}} - o(1)}}} \le \eps
$$
In particular, if we choose $\gamma = \eps/(1-2^{-n})$, we get
$$
    2^{-n^2\inparen{1 - h(\gamma) - 2c \log\inparen{\frac{8e}{c}} - o(1)}} ~\ge~ 2^{-n}\,.
$$
And hence,
$$
2c \log\inparen{\frac{8e}{c}} ~\ge~ 1 - h\inparen{\frac{\eps}{1-2^{-n}}} - \frac{1}{n} - o_n(1) ~\ge~ 1 - h(\eps) - o_n(1)\,.
$$
Thus,
$$
c ~\ge~ \frac{1 - h(\eps)}{4 \log (16e / (1-h(\eps)))} - o_n(1)\,.
$$
This concludes the proof.
\end{proofof}

\section{Lower Bounds for ReLU Functions : Proof of Theorem~\ref{thm:relu}}\label{apx:proof-relu}

Our proof proceeds in a modular fashion:
Instead of directly lower bounding $\minEVdim$ for $\calH_{n,W,B}^{\relu}$, we prove a lower bound for the class obtained as linear combination of a $\poly(n)$ number of functions in $\calH_{n,W,B}^{\relu}$. Towards this goal, for any class $\calH \subseteq \bbR^{\calX}$, define
$$
\kappa \cdot \calH := \Set{\kappa h : h\in \calH} \qquad \sAND \qquad \calH^{k,A} := \Set{\sum_{i=1}^k a_i h_i : \sum_i a_i^2 \le A \sAND h_i \in \calH}
$$
\begin{proposition}\label{prop:lin-comb-dc}
For all $\calH \subseteq \bbR^{\calX}$, all distribution $\calD$ over $\calX$, and parameters $\kappa, k, A$,
\begin{itemize}
\item[(i)] $\dc_{\eps/\kappa}^{\calD,\lsq}(\calH) ~=~ \dc_{\eps}^{\calD,\lsq}(\sqrt{\kappa} \cdot \calH)$ for all $t \in \bbR$
\item[(ii)] $\dc_{\eps}^{\calD,\lsq}(\calH^{k,A}) ~\le~ \dc_{\eps/kA}^{\calD,\lsq}(\calH)$ for all $k \in \bbN$ and $A \in \bbR$
\end{itemize}
Thus, combining the two parts,
\begin{equation}\label{eqn:cor-lin-comb-dc}
\dc_{\eps}^{\calD,\lsq}(\calH^{k,A}) ~\le~ \dc_{\eps}^{\calD,\lsq}(\sqrt{kA} \cdot \calH)
\end{equation}
\end{proposition}
\begin{proof}
Part (i) follows easily by observing that square loss is quadratic in the scaling of $\calH$ (and $w \in \bbR^d$).
To establish Part (ii): Let $\calP$ be the distribution over embeddings $\varphi : \calX \to \bbR^d$ that realizes the definition of $\dc_{\eps/kA}^{\calD,\lsq}(\calH) =: d$. For any $\varphi$ and any $g = \sum_{i=1}^k a_i h_i \in \calH^{k,A}$, we have,
\begin{align*}
&\inf_{w \in \bbR^d} \ \err_{\calD,g}^{\lsq}(\inangle{w,\varphi(\cdot)}) ~=~ \frac{1}{2} \ \inf_{w \in \bbR^d} \ \Ex_{x \sim \calD} \inparen{\sum_{i=1}^k a_i h_i(x) - \inangle{w,\varphi(x)}}^2\\
&~=~ \frac{1}{2} \ \inf_{w_1, \ldots, w_k \in \bbR^d} \ \Ex_{x \sim \calD} \inparen{\sum_{i=1}^k a_i h_i(x) - \inangle{\sum_i a_i w_i,\varphi(x)}}^2 \qquad \ldots (\text{setting } w = \sum_i a_i w_i)\\
&~\le~ \frac{k}{2} \cdot \sum_{i=1}^k a_i^2 \cdot \inf_{w_i \in \bbR^d} \ \Ex_{x \sim \calD} \inparen{h_i(x) - \inangle{w_i,\varphi(x)}}^2\\
&~=~ k \cdot \sum_{i=1}^k a_i^2 \cdot \inf_{w_i \in \bbR^d} \err_{\calD,h_i}^{\lsq}(\inangle{w_i, \varphi(\cdot)})
\end{align*}
The proof concludes by taking an expectation over $\varphi \sim \calP$,
$$
\Ex_{\varphi \sim \calP} \insquare{\inf_{w \in \bbR^d} \err_g^{\calD,\lsq}(\inangle{w,\varphi(\cdot)})} ~\le~ k \sum_i a_i^2 \cdot \frac{\eps}{kA} ~\le~ \eps\,.
$$
\end{proof}


\begin{proofof}{\theoremref{thm:relu}}
We will show a lower bound on the SQ-dimension of a class of linear combinations of ReLU neurons. In order to do, we consider
for any odd $a$, the univariate function
$$
\psi_a(z) := -1 + [z+a]_+ + \sum_{i=1}^{a-1} 2 \cdot (-1)^i \cdot [z+a-2i]_+ - [z-a]+
$$
See \figureref{fig:psi-function} for an illustration of this function. We now consider the class
$$
\calH_n^{\mathrm{zig}} := \Set{\psi_a(\inangle{w,x}) : w \in \bbR^n, \|w\|_2 = n \text{ for } a = 6n^2+1}\,.
$$

\noindent The key idea for showing a lower bound on $\SQdim^{\calD}(\calH_n^{\mathrm{zig}})$ is the following proposition that can be inferred\footnote{Part (i) is verbatim. For Part (ii), we can first infer the desired claim for a fixed $u$ and a random $v$, and then take an expectation over $u$.} from Proposition 4.2 in \cite{yehudai19power}; we skip the details.
\begin{proposition}[Prop 4.2 in \cite{yehudai19power}] There exist constants $c, c' > 0$ such that, for $a = 6n^2+1$ and $\calD$ being the standard $n$-variate Gaussian distribution,
\begin{itemize}
    \item [(i)] For all $w \in \bbR^n$ with $\|w\| = n$, it holds that $\|\psi_a(\inangle{w,x})\|_{\calD} \ge c'$.
    \item [(ii)] For $u, v$ sampled uniformly at random from $\Set{w : \|w\|=n}$,
    $$\Ex_{u, v} \inparen{\Ex_{x\sim\calD} \psi_a(\inangle{u,x}) \psi_a(\inangle{v,x})}^2 ~\le~ \exp(-c n)\,.$$
\end{itemize}
\end{proposition}

\noindent Thus, if we sample $u_1, \ldots, u_t$ randomly from $\Set{w : \|w\| = n}$, then (via Markov's inequality and a union bound) we will have with probability at least $1/2$ that,
$$
\text{for all } i \ne j \quad : \quad \inabs{\Ex_{x\sim\calD}\  \psi_a(\inangle{u_i,x}) \psi_a(\inangle{u_j,x})} \le t^2 \cdot \exp(-c n)\,.
$$
In particular, for $t := \exp(c n/3)/2$ there exist $u_1, \ldots, u_t \in \bbR^n$ such that $\|u_i\| = n$ and all pairwise correlations $|\inangle{\psi_a(\inangle{u_i,x}), \psi_a(\inangle{u_j,x})}_{\calD}| \le \exp(-cn/3)/4 \le 1/2t$. Thus, we get that, $\SQdim^{\calD}(\calH_n^{\mathrm{zig}}) \ge \exp(\Omega(n))$. Note however that there is a slight technicality here in that $\calH_n^{\mathrm{zig}}$ is not a normalized hypothesis class. But observe that all hypotheses in $\calH_n^{\mathrm{zig}}$ have the same norm $\|\cdot\|_{\calD}$ which is at least $c'$. Thus, we can make $\calH_n^{\mathrm{zig}}$ normalized by scaling it by $\|\psi_a(\inangle{u,\cdot})\|_\calD^{-1} \le 1/c'$. This would increase the correlations by a factor of at most $(1/c')^2$. 
Thus, from \corollaryref{cor:dcl2-lowerbound-SQ}, we have that $\dc_{\eps}^{\calD,\lsq}(\calH_n^{\mathrm{zig}}) \ge (1-4\eps) \exp(\Omega(n))$.

\begin{figure}[t]
	\centering
	\begin{tikzpicture}[xscale=0.4,yscale=0.9]
	\path[latex-latex,line width=0.4pt,black]
	(-8,0) edge (8,0)
	(0,-1.5) edge (0,1.5);
	
	\draw[-,line width=1.2pt,myBlue]
	(-7.7,-1) -- (-5,-1) -- (-3,1) -- (-1,-1) -- (1,1) -- (3,-1) -- (5,1) -- (7.7,1);
	
	\foreach \x in {-7,...,7}
	{ \draw[thick] (\x,-0.05) edge (\x,0.05); }
	\foreach \y in {-1,0,1}
	{ \draw[thick] (-0.05,\y) edge (0.05,\y); }
	\end{tikzpicture}
	\caption{Plot of $\psi_5 : \bbR \to \bbR$}
	\label{fig:psi-function}
\end{figure}

Observe that every $g \in \calH^{\mathrm{zig}}_n$ can be written as a linear combination of $6n^2+3$ ReLU neurons of the form $[\inangle{w,x}+b]_+$, where $\|w\| \le n$ and $|b| \le 6n^2 + 1 < 7n^2$ (where we can simulate the constant term with $w = 0$), where each coefficient in the linear combination is at most $2$. Thus, in our notation, $\calH^{\mathrm{zig}}_n \subseteq (\calH^{\relu}_{n,n,7n^2})^{k,A}$ for $k = 6n^2+3$ and $A = 4(6n^2+3)$. Thus, we get,
\begin{align*}
    \exp(\Omega(n)) &~\le~ \dc_{\eps}^{\calD,\lsq}(\calH^{\mathrm{zig}}_n)\\
    &~\le~ \dc_{\eps}^{\calD,\lsq}((\calH^{\relu}_{n,n,7n^2})^{k,A})\\
    &~\le~ \dc_{\eps}^{\calD,\lsq}(\sqrt{kA} \cdot \calH^{\relu}_{n,n,7n^2}) \qquad \ldots \text{(from \propositionref{prop:lin-comb-dc})}\\
    &~\le~ \dc_{\eps}^{\calD,\lsq}(\calH^{\relu}_{n,14n^3,98n^4})
\end{align*}
where the last step uses that $\kappa \cdot \calH^{\relu}_{n,W,B} = \calH^{\relu}_{n,\kappa W,\kappa B}$ (which follows from the homogeneity of ReLU).
This completes the proof.
\end{proofof}

\end{document}

%% file: my-jmlr-fix.tex
\newcommand*{\propositionrefname}{Proposition}
\newcommand*{\propositionsrefname}{Propositions}
\newcommand*{\propositionref}[1]{%
  \objectref{#1}{\propositionrefname}{\propositionsrefname}{}{}}

\newenvironment{proofof}[1]%
  {%
   \par\noindent{\bfseries\upshape Proof of #1\ }%
  }%
  {\jmlrQED}

%% file: preamble.tex
\usepackage{bbm}
\usepackage{mathrsfs}
\usepackage{nicefrac}

\definecolor{Gred}{RGB}{219, 50, 54}
\definecolor{Ggreen}{RGB}{60, 186, 84}
\definecolor{Gblue}{RGB}{72, 133, 237}
\definecolor{Gyellow}{RGB}{247, 178, 16}
\definecolor{ToCgreen}{RGB}{0, 128, 0}
\definecolor{myGold}{RGB}{231,141,20}
\definecolor{myBlue}{rgb}{0.19,0.41,.65}
\definecolor{myPurple}{RGB}{175,0,124}

\providecommand{\tdotoggle}{1}
\newcommand{\mytodo}[1]{\ifnum\tdotoggle=1{#1}\fi}

\newcommand{\tableoftodos}{\ifnum\tdotoggle=1 \listoftodos[Comments/To Do's] \fi}

\newcommand{\remove}[1]{}

\usepackage{tikz}
\usetikzlibrary{arrows.meta}
\usepackage[framemethod=tikz]{mdframed}
\mdfsetup{
	roundcorner=4pt,
	nobreak=true,
	linewidth=.5,
	innerleftmargin=10pt,
	innerrightmargin=10pt,
	innertopmargin=10pt,
	innerbottommargin=10pt,
	skipabove=12,skipbelow=5pt
}

\newcommand{\inbrace}[1]{\left \{ #1 \right \}}
\newcommand{\inparen}[1]{\left ( #1 \right )}
\newcommand{\insquare}[1]{\left [ #1 \right ]}
\newcommand{\inangle}[1]{\left \langle #1 \right \rangle}

\newcommand{\inabs}[1]{\begin{vmatrix} #1 \end{vmatrix}}

\newcommand{\Set}[1]{\inbrace{#1}}

\DeclareMathOperator*{\sign}{sign}
\DeclareMathOperator*{\argmin}{argmin}

\DeclareMathOperator*{\Ex}{\mathbb{E}}

\DeclareMathOperator*{\Prob}{\mathsf{Pr}}

\newcommand{\poly}{\mathrm{poly}}
\newcommand{\wtilde}[1]{\widetilde{#1}}
\newcommand{\what}[1]{\widehat{#1}}

\newcommand{\sbit}{\Set{1, -1}}
\newcommand{\eps}{\varepsilon}

\newcommand{\sAND}{\ \mathrm{and} \ }

\newcommand{\bbH}{{\mathbb H}}
\newcommand{\bbN}{{\mathbb N}}

\newcommand{\bbR}{{\mathbb R}}

\newcommand{\calA}{\mathcal{A}}
\newcommand{\calB}{\mathcal{B}}
\newcommand{\calC}{\mathcal{C}}
\newcommand{\calD}{\mathcal{D}}

\newcommand{\calH}{\mathcal{H}}

\newcommand{\calL}{\mathcal{L}}

\newcommand{\calO}{\mathcal{O}}
\newcommand{\calP}{\mathcal{P}}

\newcommand{\calX}{\mathcal{X}}
\newcommand{\calY}{\mathcal{Y}}

\newcommand{\scrD}{\mathscr{D}}

\newcommand{\ind}{\mathbbm{1}}

\newcommand{\abs}[1]{\left\lvert{#1}\right\rvert}

%% file: macros.tex
\newcommand{\dc}{\mathsf{dc}}
\newcommand{\mc}{\mathsf{mc}}
\newcommand{\vc}{\mathsf{VC}\text{-}\mathsf{dim}}
\newcommand{\VCdim}{\vc}
\newcommand{\SQdim}{\mathsf{SQ}\text{-}\mathsf{dim}}
\newcommand{\minEVdim}{\mathsf{minEV}\text{-}\mathsf{dim}}
\newcommand{\Lin}{\mathsf{Lin}}
\newcommand{\Ker}{\mathsf{Ker}}
\newcommand{\gLin}{\mathsf{gLin}}
\newcommand{\gKer}{\mathsf{gKer}}

\newcommand{\ERM}{\textsc{Erm}}
\newcommand{\err}{\calL} 

\newcommand{\relu}{\mathrm{relu}}

\newcommand{\pt}{\mathrm{pt}}

\newcommand{\signrank}{\mathsf{sign\text{-}rank}}
\newcommand{\rank}{\mathsf{rank}}

\newcommand{\lzeroone}{{\ell_{\mathrm{0}\text{-}\mathrm{1}}}}
\newcommand{\lmargin}{{\ell_{\mathrm{mgn}}}}

\newcommand{\lsq}{{\ell_{\mathrm{sq}}}}
\newcommand{\lhinge}{{\ell_{\mathrm{hinge}}}}